\documentclass[11pt,oneside,letterpaper]{article}

\usepackage[mode=buildnew]{standalone}
\usepackage{eurosym}
\usepackage{graphicx,amsmath,amsfonts}
\usepackage{setspace}
\usepackage{amssymb}
\usepackage{dsfont}
\usepackage{amsmath}
\usepackage{amsfonts}
\usepackage{ifthen}
\usepackage{mathtools}
\usepackage{graphicx}
\usepackage{enumitem}
\usepackage{color}
\usepackage{framed}
\usepackage[margin=1in]{geometry}
\usepackage{natbib}
\usepackage[dvipsnames]{xcolor}
\usepackage[colorlinks=true, urlcolor=Mahogany, linkcolor=Mahogany, citecolor=Mahogany]{hyperref}
\usepackage{fp}
\usepackage[capitalize]{cleveref}
\usepackage{amsthm}
\usepackage{caption}
\usepackage{subcaption}
\usepackage{palatino}
\usepackage{cuted,balance}
\usepackage{amssymb}
\usepackage{soul}
\usepackage{graphicx,amsmath,amsfonts}
\usepackage{float}
\usepackage{bbm}

\usepackage[ruled]{algorithm2e}

\newtheorem{proposition}{Proposition}
\newtheorem{theorem}{Theorem}
\newtheorem{lemma}{Lemma}

\newtheorem{remark}{Remark}

\allowdisplaybreaks

\let\originaleqref\eqref
\renewcommand{\eqref}{\originaleqref}
\graphicspath{{figs/}}
\newcommand{\E}{\mathbb{E}}
\newcommand{\Prob}{\mathbb{P}}
\newcommand{\R}{\mathbb{R}}

\newcommand{\indic}{\mathbf{1}}
\newcommand{\ddm}{\ifmmode\mathrm{DDM}\else DDM\fi}
\newcommand{\name}{\ifmmode\mathrm{DDMPO}\else DDMPO\fi}
\newcommand{\ddmpo}{\ifmmode\mathrm{DDMPO}\else DDMPO\fi}
\newcommand{\dpo}{\ifmmode\mathrm{DPO}\else DPO\fi}
\newcommand{\ipo}{\ifmmode\mathrm{IPO}\else IPO\fi}
\newcommand{\rlhf}{\ifmmode\mathrm{RLHF}\else RLHF\fi}
\newcommand{\llm}{\ifmmode\mathrm{LLM}\else LLM\fi}
\newcommand{\bt}{\ifmmode\mathrm{BT}\else BT\fi}
\newcommand{\btmodel}{Bradley--Terry}

\newcommand{\Fstar}{F^{*}}
\newcommand{\bhat}{\widehat b_n}
\newcommand{\btil}{\widetilde B_n}
\newcommand{\muhat}{\widehat\mu_n}
\newcommand{\muhatpi}{\widehat\mu_n^{\mathrm{PI}}}
\newcommand{\thetahat}{\widehat{\theta}_n}
\newcommand{\Lhat}{\widehat L_n}
\newcommand{\Lb}{L_b}
\newcommand{\Nbeta}{N_\beta}
\newcommand{\Dbeta}{D_\beta}
\newcommand{\wbeta}{w_\beta}
\newcommand{\ck}{c_k}

\newcommand{\cX}{\mathcal{X}}

\newcommand{\cD}{\mathcal{D}}

\newcommand{\cN}{\mathcal{N}}

\newcommand{\Var}{\operatorname{Var}}

\newcommand{\vhat}{\widehat v}

\begin{document}

\title{Response Time Enhances Alignment with Heterogeneous Preferences}

\author{Federico Echenique\thanks{Department of Economics, University of California, Berkeley} \and 
Alireza Fallah\thanks{Department of Computer Science and Ken Kennedy Institute, Rice University} \and 
Baihe Huang \thanks{Departments of Electrical Engineering and Computer Sciences, University of California, Berkeley} \and
Michael I. Jordan\thanks{Departments of Electrical Engineering and Computer Sciences and Statistics, University of California, Berkeley; Inria Paris}}

\date{May 7, 2026}

\maketitle

\sloppy

\begin{abstract}
Aligning large language models (LLMs) to human preferences typically relies on aggregating pooled feedback into a single reward model. However, this standard approach assumes that all labelers share the same underlying preferences, ignoring the fact that real-world labelers are highly heterogeneous and usually anonymous. Consequently, relying solely on binary choice data fundamentally distorts the learned policy, making the true population-average preference unidentifiable. To overcome this critical limitation, we demonstrate that augmenting preference datasets with a simple, secondary signal—the user's response time—can restore the identifiability of the population's average preference. By modeling each decision as a Drift-Diffusion Model (DDM), we introduce a novel, consistent estimator of heterogeneous preferences that successfully corrects the distortions of standard choice-only labels. We prove that our estimator asymptotically converges to the true average preference even in extreme cases where each anonymous labeler contributes only a single choice. Empirically, across both synthetic and real-world datasets, our method consistently outperforms standard baselines that otherwise fail and plateau at a bias floor. Because response times are essentially free to record and require zero user tracking or identification, our results bring promises and open up new opportunities for future data-collection pipelines to improve the social benefit without requiring user-level identifiers or repeated elicitations.

\end{abstract}

\section{Introduction}
\label{sec:intro}

Large language models (LLMs) are aligned to human preferences through pairwise-comparison data~\citep{openai2023gpt4,llama3_2024,gemini2023,anthropic2024claude3}. The standard methodology, reinforcement learning from human feedback (\rlhf{})~\citep{christiano2017deep,stiennon2020learning,ouyang2022training,bai2022training}, fits a Bradley--Terry reward model to binary preferences and then optimizes the policy against that reward. Direct Preference Optimization (\dpo{})~\citep{rafailov2023dpo} and related objectives~\citep{azar2024general,ethayarajh2024kto,meng2024simpo,hong2024orpo} change the optimization layer, but they keep the same single-reward statistical backbone.

These methods typically assume a unique preference shared by all labelers. Real-world datasets, however, pool anonymous annotators with highly heterogeneous tastes, values, and expertise. Recent work shows that this mismatch can badly distort the learned policy and can even make the population-average preference unidentifiable from choices alone~\citep{siththaranjan2024distributional,shirali2025directhetero,golz2025distortion,chidambaram2024dpohetero,park2024heterorlhf}.

One solution is to estimate each labeler's preferences separately, but this presents two significant challenges. First, large-scale datasets are often anonymous or contain too few comparisons per labeler. Second, even with perfect access to each labeler's individual preferences, how are we to aggregate them into a single aggregate reward?

We deal with the first problem by jointly modeling a group of labelers. Each individual is assumed to follow the \emph{Drift Diffusion Model} (\ddm{}), the dominant procedural choice model in psychology and neuroscience~\citep{ratcliff1978theory,ratcliff2008diffusion,gold2007neural,fehr2011neuroeconomic,clithero2018improving}. The DDM assumes a Brownian motion with latent drift $V$ absorbed at $\pm b$, where $V$ is the signed utility difference between the two choices, and $b$ is the decision threshold. Different labelers have different utilities and thus idiosyncratic drifts, but they share the boundary parameter $b$.

We deal with the second problem using response time to gauge utility intensity. The problem of uncovering a common aggregate reward has a long history in social choice theory. The 18th-century philosopher Marquis de Condorcet famously noted that pairwise comparisons cannot always be aggregated into a single winning choice. The modern version of Condorcet's paradox is Arrow's impossibility theorem, which states that no voting procedure satisfies basic efficiency and fairness properties. 

Our solution to the paradoxes of Condorcet and Arrow lies in the measurement of utility. \citet{d1977equity} show that Arrow's theorem stems from the assumption that utilities are entirely ordinal, making them incomparable. By introducing degrees of comparability, they demonstrate that a utilitarian criterion (such as an average reward) can emerge. This comparability, however, requires a mechanism to measure the intensity of preferences. Our approach rests precisely on such a mechanism.

The \ddm{} predicts that response time and utility are related. An increase in the intensity of utility translates into a larger drift, and therefore a shorter response time. Indeed, marginalizing out response time reduces the \ddm{} directly to the \btmodel{} model. Recent work shows that, when $b$ is known, response times recover preference parameters at the parametric $1/n$ rate \citep{echenique2025response}. Our departure is the practically relevant case in which the boundary is \emph{common but unknown}. We therefore estimate the latent boundary and the population-average drift jointly from anonymous choices and response times.
Our contributions are three-fold.
\begin{itemize}[leftmargin=1.5em,itemsep=1pt,topsep=1pt]
    \item \textbf{Unbiased drift and consistent boundary estimation.} Given the boundary $b$, we derive an unbiased estimator of the utilitarian aggregate reward, defined as the average drift $\mu=\E[V]$. The estimator is a weighted average of response times, with closed-form weights $w_b(\cdot)$. From the large-$\lambda$ tail of the response-time Laplace transform, we then construct a Richardson-extrapolated consistent estimator $\widetilde B_n$ of $b$, whose population bias is $O(\lambda_n^{-1})$---improving the one-scale $O(\lambda_n^{-1/2})$ bias under bounded drifts.
    \item \textbf{Consistent average preference in the linear model.} In the heterogeneous linear preference model with per-labeler coefficient $\theta_i\in\R^d$ and drift $V_i=\psi_i^\top\theta_i$, combining the plug-in response-time weight $w_{\widetilde B_n}(\cdot)$ with ordinary least squares yields a consistent estimator of the population-average preference vector $\theta^\star=\E[\theta_i]$, with asymptotically negligible plug-in error from $\widetilde B_n$.
    \item \textbf{Empirical validation.} On synthetic data over diverse latent-preference priors and on an intertemporal-choice dataset from \citet{amasino2019amount}, reanalyzed by \citet{echenique2025response}, the response-time plug-in matches an oracle that knows the boundary, while the choice-only baseline plateaus at a heterogeneity-induced bias floor.
\end{itemize}

To sum up, the use of response time in an anonymous, heterogeneous population of decision makers allows for the recovery of a population-average preference that serves as a utilitarian aggregate objective. While response times are not yet logged in standard preference-learning datasets for LLM alignment, the required instrumentation is minimal; our results strongly argue for recording this essential free signal in future data-collection pipelines.

\subsection{Related Work}
\label{sec:related}

\paragraph{Reinforcement learning from human feedback.}
\rlhf{} fits a reward model to binary preferences via the Bradley--Terry likelihood and then optimizes the policy against that reward~\citep{christiano2017deep,stiennon2020learning,ouyang2022training,bai2022training,ziegler2019fine,bradley1952rank,schulman2017proximal}. \dpo{} and related objectives replace the optimization layer while keeping the same single-reward statistical target~\citep{rafailov2023dpo,azar2024general,zhao2023slic,ethayarajh2024kto,meng2024simpo,hong2024orpo,tang2024generalized}. A complementary line of work studies reward-model over-optimization, robustness, and the gap between online and offline alignment~\citep{gao2023scaling,zhu2023fine,zhu2023principled,eisenstein2023helping,chowdhury2024provably,liu2024provably,tang2024understanding,chang2024dataset}. Response times do not enter the statistical design of these methods.

\paragraph{Heterogeneous preferences and the Bradley--Terry distortion.}
Recent work shows that a single Bradley--Terry reward can be badly misspecified under heterogeneous feedback. Hidden-context analyses imply Borda-style aggregation rather than mean aggregation~\citep{siththaranjan2024distributional}; direct alignment without per-labeler information faces a consistency-sample-efficiency tradeoff~\citep{shirali2025directhetero}; and social-choice analyses quantify severe worst-case distortion for Bradley--Terry-based alignment, with Nash Learning from Human Feedback emerging as distortion-optimal in the worst case~\citep{golz2025distortion,munos2023nash}. Other approaches enrich the learner with mixtures, personalization, or explicit value axes~\citep{chidambaram2024dpohetero,park2024heterorlhf,sorensen2023valuekaleidoscope}. We instead keep the learner simple and add response time as an orthogonal signal, with no labeler identifier at all.

\paragraph{Drift-diffusion model and response-time-based inference.}
The \ddm{} is a standard model of perceptual and value-based binary choice~\citep{ratcliff1978theory,ratcliff1998modeling,ratcliff2008diffusion,gold2007neural,krajbich2010visual,fehr2011neuroeconomic,strzalecki2025stochastic}. It yields explicit formulas for the joint law of choice and first-passage time and has been used in food choice, intertemporal choice, advertising, and laboratory decision tasks~\citep{krajbich2010visual,amasino2019amount,chiong2024combining,clithero2018improving,fudenberg2018speed,fudenberg2020testing}. On the inference side, prior work develops nonparametric tests, ordinal inference from response times, and $1/n$-rate estimation when the boundary is known~\citep{fudenberg2020testing,alosferrer2021time,echenique2025response}; recent ML work also studies response times in robust preference learning and linear bandits~\citep{sawarni2025preference,li2024rtbandits}. To our knowledge, our paper is the first response-time-based preference-learning work that simultaneously targets the population-average preference in an anonymous heterogeneous labeler pool and provides a consistent estimator when the common DDM boundary is unknown and must be estimated from the same response-time data.

\section{Preliminaries}
\label{sec:prelim}

\paragraph{Preference learning.}
A preference dataset is a set of tuples $\cD=\{(x_i,y_i,z_i)\}_{i=1}^n$ in which $(x_i,y_i) \in \mathcal{X}^2$ is a pair of alternatives (e.g.\ two LLM responses to the same prompt) and $z_i\in\{+1,-1\}$ encodes whether labeler $i$ prefers $x_i$ ($z_i=+1$) or $y_i$ ($z_i=-1$). The standard \btmodel{} model assumes a latent reward $r_\theta:\cX\to\R$ and posits
\begin{align*}
    \Prob(z_i=+1\mid x_i,y_i) = \sigma\!\big(r_\theta(x_i)-r_\theta(y_i)\big), \qquad \sigma(u)=\tfrac{1}{1+e^{-u}}.
\end{align*}
Classical \rlhf{}~\citep{ouyang2022training} first fits $\theta$ by maximum likelihood and then optimizes a KL-regularized objective against $r_\theta$. 
Variants such as \dpo{}~\citep{rafailov2023dpo}, \ipo{}~\citep{azar2024general}, KTO~\citep{ethayarajh2024kto}, SimPO~\citep{meng2024simpo}, and ORPO~\citep{hong2024orpo} retain the single-reward assumption but differ in the loss or the parameterization.

\paragraph{Alignment under heterogeneous preferences.}
Real-world preference datasets pool annotations from many labelers with potentially different preferences. Label $i$ comes from an individual labeler with utility $u_i:\cX\to\R$, drawn i.i.d.\ from a population distribution. Individual $i$ is never observed again. We seek to learn the \emph{average utility,} which corresponds to the utilitarian criterion we discussed in the introduction. \citet{siththaranjan2024distributional,shirali2025directhetero,golz2025distortion} show that the single-reward likelihood is misspecified under heterogeneity: even with infinitely many samples, the \btmodel{} MLE converges to the \emph{Borda-rule} aggregate rather than the average-utility. The gap between the utilitarian and Borda aggregates can be large (see \Cref{sec:util-aggregation} for details), and a utilitarian goal has stronger foundations when utility is cardinally meaningful. Consequently, a model trained with \rlhf{} or \dpo{} on heterogeneous data need not rank alternatives in the order dictated by the utilitarian criterion $\mathrm{AvgUtil}(x):=\E[u_i(x)]$. \citet{shirali2025directhetero} prove an incompatibility result between consistency and sample efficiency when per-labeler information is unavailable, motivating the need to go beyond binary choice data: hence response times. 
Response times are one such signal, and they are the focus of this paper.

\paragraph{Drift-diffusion model.}
The Drift diffusion model (\ddm{})~\citep{ratcliff1978theory,ratcliff2008diffusion} is the dominant model in psychology and neuroscience for modeling choice and response time. The \ddm{} assumes that information about $x_i$ and $y_i$ arrives over time according to a Brownian motion $W_\tau$ with drift $v$. When $W_{\tau}$ first reaches a boundary $b$ or $-b$ a decision is made:
\begin{align}\label{eq:ddm-sde}
W_{\tau} = B_{\tau} + v \cdot \tau, \quad T=\inf\{\tau>0: |W_{\tau}|=b\}, \quad Z=2 \cdot \indic(W_T=b)-1\in\{+1,-1\},
\end{align}
where $(B_{\tau})$ is a standard Brownian motion with $B_0=0$. The sign $Z$ encodes the choice between $x_i$ and $y_i$, the stopping time $T$ is the response time\footnote{An extended version of the \ddm{} allows the time it takes humans to process the alternatives, called \textit{non-decision time}, to be included in the response time \citep{ratcliff2002estimating}. In Appendix \ref{app:no_decision_time}, we discuss how our results can be adjusted to accommodate this.}, and $v$ is the underlying utility difference between $x_i$ and $y_i$. The \ddm{} can be viewed as a decision maker learning about the relative quality of two alternatives through separate Gaussian signals. Posterior beliefs become more accurate over time, but the decision maker must choose when to stop, since waiting entails a constant marginal cost~\citep{fudenberg2018speed}.

Importantly, the \ddm{} recovers the \btmodel{} as a special case when response time is ignored. The probability of $Z=z$ in the \ddm{} is given by $\sigma(2bzv)$ \citep{echenique2025response}. In other words, the \ddm{} provides a joint model of choice and response time that is consistent with the \btmodel{} and has strong support in the psychology literature \citep{ratcliff2008diffusion, ratcliff2016diffusion}.
Furthermore, under fixed $(v,b)$, it is known that $Z$ and $T$ are conditionally independent given $(v,b)$, and $\E[Z\mid v,b]=\tanh(bv)$ \citep{echenique2025response,fudenberg2020testing}.

\paragraph{Roadmap.}
We assume that individual labelers follow the DDM model, each with their idiosyncratic utility and, thus, an idiosyncratic drift. We seek to estimate the mean drift, which corresponds to the utilitarian criterion. We assume that all labelers share a common boundary $b$, which is tantamount to assuming a common temperature in \btmodel{}. In Appendix \ref{app:boundary_heterogeneity}, we discuss what happens if this assumption does not hold and propose possible adjustments.

In \Cref{sec:estimators}, we start off with a simple scenario: every observation comes from a fresh labeler with an unknown latent drift $V \sim \Fstar$ (i.i.d.). In other words, for a labeler $i$ with drift $v_i$, we only get to see $(z_i, t_i)$ generated by Eq.~\eqref{eq:ddm-sde}. We just assume bounded support, meaning there is a known constant $M < \infty$ such that $|V| \le M$ almost surely\footnote{We relax this assumption to a bounded exponential-moment condition in Appendix~\ref{subsec:light-tailed-drifts}, which holds for Gaussian drift.}. One can interpret this as asking all labelers to choose between the same two alternatives, $x$ and $y$. Here, $v_i$ is simply $u_i(x) - u_i(y)$, where $u_i(\cdot)$ represents labeler $i$'s utility. Our main goal here is to learn the average utility difference, i.e., $\mu:= \mathbb{E}_{V \sim \Fstar}[V]$.

In \Cref{sec:linear}, we extend this analysis to a contextual linear setting with multiple alternatives. More precisely, labeler $i$ is offered two alternatives $(x_i, y_i) \in \mathcal{X}^2$. We assume that labeler $i$'s utility is linear and given by $u_i(x) = \phi(x)^\top \theta_i$, where $\theta_i \in \mathbb{R}^d$ is the labeler's unknown latent preference vector, drawn from a distribution $G^*$, and $\phi: \mathcal{X} \to \mathbb{R}^d$ is the known feature map. 
In this case, the drift becomes $v_i = \left(\phi(x_i) - \phi(y_i)\right)^\top \theta_i$, and our goal is to learn the population-average preference vector $\theta^\star := \mathbb{E}_{\theta \sim G^*}[\theta]$. Note that we never observe the labeler's identity.

\section{Estimating the Common Boundary and the Average Drift}
\label{sec:estimators}
Our goal is to estimate the utilitarian objective, the average drift $\mu = \mathbb{E}_{V \sim \Fstar}[V]$, assuming access to the dataset ${(z_i,t_i)}_{i=1}^n$, where $z_i$ and $t_i$ denote the choice and response-time data, respectively, generated by a \ddm{} with drift $v_i \sim \Fstar$. We present a two-stage construction. First, we design an unbiased estimator of $\mu$, \textit{assuming that the boundary parameter $b$ is known}. Next, we develop a \emph{consistent} Richardson-extrapolated estimator of the common boundary.

\subsection{Unbiased estimation of the average drift given the boundary}
\label{sec:unbiased-drift}
For $k\ge 0$ and $b>0$, write $\ck(b):=(2k+1)b$ and for $t>0$ define the ratio-of-series weight
\begin{align}\label{eq:wb-def}
    w_b(t)
    \;:=\;
    \frac{\displaystyle\sum_{k=0}^\infty \left(\frac{\ck(b)^2}{t}-1\right) \exp\!\left(-\frac{\ck(b)^2}{2t}\right)\!}
         {\displaystyle\sum_{k=0}^\infty (-1)^k \ck(b)\exp\!\left(-\frac{\ck(b)^2}{2t}\right)}.
\end{align}

\begin{theorem}[Unbiased estimator of the average drift given the boundary]
\label{thm:unbiased}
Fix the value of the boundary $b>0$. Under the \ddm{} data-generating process of Section~\ref{sec:prelim}, the estimator
\begin{align*}
\muhat \;=\; \frac{1}{n}\sum_{i=1}^n z_i\, w_b(t_i)
\end{align*}
satisfies $\E\!\left[\muhat\right]=\mu$ for every drift distribution $\Fstar$. 
Moreover, $w_b$ is the \emph{unique} smooth function satisfying $\E[Z w_b(T)\mid V=v]=v$ for all $v$.
\end{theorem}
\textit{Proof sketch:} 
To show that $\E[Z w_b(T)\mid V=v]=v$, we first recall that, conditional on the drift $V=v$, the choice $Z$ and response time $T$ are independent. Therefore,
\begin{equation*}
\E[Z w_b(T)\mid V=v] = \E[Z\mid V=v] \cdot \E[ w_b(T)\mid V=v].    
\end{equation*}
As stated in \Cref{sec:prelim}, the first expectation is known to be $\tanh(bv)$. Hence, it remains to characterize the second expectation. By Girsanov’s theorem, this problem can be reduced to computing the Laplace transform of $w_b(t)f_0(t)$, where $f_0(\cdot)$ is the density of the response time when the drift is zero. The proof then concludes by evaluating this Laplace transform. 
A proof is provided in Appendix~\ref{app:proof-unbiased}. \qed

\subsection{Consistent estimation of the common unknown boundary}
\label{sec:consistent-boundary}

The weight function \eqref{eq:wb-def} requires knowledge of the boundary parameter $b$: We now turn to a consistent estimator of $b$. To this end, we use the empirical Laplace transform of the response times, given by
\begin{align}\label{eq:Lhat-def}
    \Lhat(\lambda) \;:=\; \frac{1}{n}\sum_{i=1}^n e^{-\lambda t_i}.
\end{align}
To see why this is useful, let us consider for a moment the homogeneous case with fixed drift $V=v$. Note that, as $n \to \infty$, the empirical Laplace transform converges to its population counterpart, $\Lb(\lambda):=\E[e^{-\lambda T}]$, which is given by \citep[see][Lemma 4]{echenique2025response}:
\begin{equation} \label{eqn:laplace_transform_drift}
\Lb(\lambda) = \frac{\cosh(bv)}{\cosh\bigl(b\sqrt{2\lambda+v^2}\bigr)}.
\end{equation}
Consequently, as $\lambda \to \infty$, we see that 
$\frac{-\log \Lb(\lambda)}{\sqrt{2\lambda}}$
converges to $b$, allowing us to recover the boundary parameter.
The next theorem formalizes this idea while accounting for raters' heterogeneity in $v$.
\begin{theorem}
\label{thm:one-scale}
Assume $|V|\le M$ almost surely. Let $(\lambda_n)_{n\ge 1}$ be any deterministic sequence with
\begin{align}\label{eq:lambda-rate}
    \lambda_n\to\infty \qquad\text{and}\qquad \sqrt{\lambda_n}=o(\log n),
\end{align}
e.g.\ $\lambda_n=(\log n)^{3/2}$. Define
\begin{align}\label{eq:bhat-def}
    \bhat \;:=\; -\frac{\log \Lhat(\lambda_n)}{\sqrt{2\lambda_n}}.
\end{align}
Then, $\bhat \xrightarrow{P} b$.
\end{theorem}

\textit{Proof sketch:} 
The proof first formalizes the intuition stated earlier, namely that
$\lim_{\lambda\to\infty} \frac{-\log \Lb(\lambda)}{\sqrt{2\lambda}}=b$
in the case of heterogeneous drifts. We do so by establishing upper and lower bounds on $\Lb(\lambda)$ of the form
\begin{equation*}
2e^{bM}e^{-b\sqrt{2\lambda}} \;\ge\; \Lb(\lambda) \;\ge\; e^{-b\sqrt{M^2+2\lambda}}.
\end{equation*}
Next, we show how to connect the empirical approximation $\Lhat(\lambda_n)$ with its population counterpart $\Lb(\lambda_n)$. 
For a fixed value of $\lambda$, $\Lhat(\lambda)$ converges almost surely to $\Lb(\lambda)$. The difficulty here is that $\lambda_n$ increases with $n$. To deal with this problem, we show that, under condition \eqref{eq:lambda-rate}, the Lindeberg--Feller Central Limit Theorem implies that 
\begin{equation*}
\sqrt{n} \cdot \frac{\Lhat(\lambda_n)-\Lb(\lambda_n)}{\sqrt{\Var(\exp(-\lambda_n T))}}
\end{equation*}
converges in distribution to $\mathcal{N}(0,1)$. From here, we establish that $\Lhat(\lambda_n)/\Lb(\lambda_n)\xrightarrow{P} 1$. The full proof is provided in Appendix~\ref{app:proof-one-scale}. \qed

\paragraph{Richardson-extrapolated estimator.} 
In our analysis of the estimator $\hat b_n$, the argument based on the central limit theorem required that  $\sqrt{\lambda_n}=o(\log n)$. However, in our numerical experiments, this requirement leads to slow convergence rates. In our next result, a simple two-scale correction improves the \textit{population bias}.
\begin{theorem}\label{thm:two-scale}
Assume $|V|\le M$ almost surely and that $\{\lambda_n\}_n$ satisfies \eqref{eq:lambda-rate}. Define
\begin{align}\label{eq:btil-def}
    \btil
    \;:=\;
    \frac{\log \Lhat(\lambda_n)-\log \Lhat(4\lambda_n)}{\sqrt{2\lambda_n}}.
\end{align}
Then, $\btil \xrightarrow{P} b$. Moreover, the population bias of $\btil$ improves upon that of $\bhat$ in the following sense:
\begin{equation} \label{eqn:estimators_convergence_rate}
\left |
~\frac{\log \Lb(\lambda)-\log \Lb(4\lambda)}{\sqrt{2\lambda}} - b ~
\right |
=
O\!\left(\frac{1}{\lambda}\right),
\qquad
\left |
~-\frac{\log \Lb(\lambda)}{\sqrt{2\lambda}} - b~
\right |
=
\Omega\!\left(\frac{1}{\sqrt{\lambda}}\right).
\end{equation}
\end{theorem}
The full proof appears in Appendix~\ref{app:proof-two-scale}. The idea is based on the following simple observation. As shown in the proof, expanding $\log \Lb(\lambda)$ as a function of $\lambda$ and with $\lambda \to \infty$, yields
\begin{align}\label{eq:laplace-expansion} 
\log \Lb(\lambda) \;=\; -b\sqrt{2\lambda} + \log(2C_0) - \frac{bC_2}{2C_0\sqrt{2\lambda}} + O(\lambda^{-1}), \end{align}
where $C_0:=\E[\cosh(bV)]$ and $C_2:=\E[V^2\cosh(bV)]$.
Consequently, the one-scale estimator $\bhat$ has an $O(\lambda_n^{-1/2})$ population bias, caused by the subleading term $\log(2C_0)$ in $\log \Lb(\lambda)$. 
$\btil$ then eliminates the dominating $O(\lambda_n^{-1/2})$ bias via a one-step Richardson extrapolation~\citep{richardson1911approximate, brezinski1991extrapolation}.

\subsection{Consistent estimation of the average drift without knowing the boundary}

Returning to our main goal of estimating the utilitarian objective $\mu = \mathbb{E}_{V \sim \Fstar}[V]$, the natural next step is to plug in the estimator $\btil$ into the estimator developed in \Cref{thm:unbiased}, which assumed knowledge of the boundary $b$. Our next result shows that the resulting estimator is consistent.  The proof is provided in Appendix~\ref{app:proof-plugin-scalar}.
\begin{theorem}
\label{thm:plugin-scalar}
Assume that $|V|\le M$ almost surely and that $\{\lambda_n\}_n$ satisfies \eqref{eq:lambda-rate}. Let
\begin{align*}
    \muhatpi \;:=\; \frac{1}{n}\sum_{i=1}^n z_i\, w_{\btil}(t_i).
\end{align*}
where $\btil$ is the estimator defined in Eq.~\eqref{eq:btil-def}.
Then, we have $\muhatpi \xrightarrow{P} \mu$.
\end{theorem}
\textit{Proof sketch:}
Decompose $\muhatpi-\mu$ as $A_n+R_n$, with $A_n:=\frac{1}{n}\sum_i z_i w_b(t_i)-\mu$ and $R_n:=\frac{1}{n}\sum_i z_i\bigl(w_{\btil}(t_i)-w_b(t_i)\bigr)$. We first note that $A_n \xrightarrow{P}0$ by \Cref{thm:unbiased} and the law of large numbers. To show $R_n$ also converges to zero in probability, we establish a envelope lemma (Lemma~\ref{lem:envelope} in the appendix) that controls $|w_\beta(t)|$ by \emph{gluing together three functions}: (i) $C_K(1+t^{-1})$ as $t$ is close to zero, (ii) a constant for $t$ in a compact middle range, and (iii) an exponential bound $C_K e^{-c_K t}$ as $t$ is large, all \emph{uniformly} in $K$. Integrability of this envelope against the marginal law of $t_i$ then yields $R_n\xrightarrow{P}0$ by using the dominated convergence theorem together with the law of large numbers. \qed

\section{A Consistent Estimator for Heterogeneous Linear Preferences}
\label{sec:linear}
We next move to the more general setting described in \Cref{sec:prelim}. We assume that, for each $x \in \mathcal{X}$, user $i$'s utility is given by $u_i(x) = \phi(x)^\top \theta_i$ for some underlying latent preference vector $\theta_i \sim G^*\in \Delta(\R^d)$ and \emph{shared} feature map $\phi:\mathcal{X} \to \R^d$. In this setting, given two alternatives $x_i$ and $y_i$, the drift has the linear form $v_i=\psi_i^\top\theta_i$, where the parameter $\psi_i := \phi(x_i) - \phi(y_i) \in \R^d$ is known to the algorithm. Throughout this section, we assume that $(\psi_i,\theta_i)$ are exogenous and i.i.d. across $i$, and the DDM paths are independent. Our goal is to estimate $\theta^\star := \E_{\theta \sim G^*} [\theta]$.

We next show how the results derived in the previous section can be used to develop a consistent estimator of the mean preference vector $\theta^\star$. Let $\vhat_i \;:=\; z_i\,w_{\btil}(t_i).$
Now consider the linear regression problem in which the $i$-th observation consists of the covariate vector $\psi_i$ and the \emph{pseudo-outcome} $\vhat_i$. Define the sample covariance matrix $\widehat Q_n \;:=\; \frac{1}{n}\sum_{i=1}^n \psi_i\psi_i^\top$ and empirical cross-moment $\widehat m_n \;:=\; \frac{1}{n}\sum_{i=1}^n \psi_i\,\vhat_i$.
Then, the \textit{plug-in regression estimator} is
\begin{equation}\label{eq:thetahat}
\thetahat \;:=\; \widehat Q_n^{-1}\widehat m_n.  
\end{equation}
We further assume that the covariance matrix $Q:=\E[\psi_i\psi_i^\top]$ is positive definite and that the sample covariance matrix $\widehat Q_n$ is almost surely invertible, which holds when $n \geq d$ and the distribution of $\psi_i$'s is non-degenerate \citep{mourtada2022exact}.
\begin{theorem}
\label{thm:linear-main}
Assume that there exist finite constants $L$ and $R$ such that $\|\psi_i\|\le L$ and $\|\theta_i\|\le R$ almost surely\footnote{We relax these assumptions in Appendix~\ref{subsec:light-tailed-drifts}, thus accommodating bounded $\psi_i$ and Gaussian $\theta_i$.}, and that the sequence $\{\lambda_n\}_n$ satisfies \eqref{eq:lambda-rate}. Then, we have $ \btil \xrightarrow{P} b$ and $\thetahat \xrightarrow{P} \theta^\star$.
Consequently, for every fixed $x\in\cX$, $\phi(x)^\top \thetahat \xrightarrow{P} \phi(x)^\top \theta^\star$.
\end{theorem}
\textit{Proof sketch:}
By the weak law of large numbers applied to $\widehat Q_n$, we have $\widehat Q_n\xrightarrow{P}Q$.
Next, we establish that $\widehat m_n\xrightarrow{P}Q\theta^\star$. 
To do so, decompose $\widehat m_n = S_n+R_n$, where $S_n :=\frac{1}{n}\sum_i\psi_i z_iw_b(t_i)$ and $R_n :=\frac{1}{n}\sum_i\psi_i z_i(w_{\btil}(t_i)-w_b(t_i))$. The term $S_n$ has expectation $\E[\psi_i\,\E[z_i w_b(t_i)\mid\psi_i,\theta_i]]=\E[\psi_i\psi_i^\top\theta_i]=Q\theta^\star$ where the first equality follows from \Cref{thm:unbiased}. Moreover, the residual term $R_n\xrightarrow{P} 0$ by the same argument of \Cref{thm:plugin-scalar}. 
Therefore, the continuous mapping theorem gives
$\thetahat=\widehat Q_n^{-1}\widehat m_n\xrightarrow{P}Q^{-1}Q\theta^\star=\theta^\star$. 
The full proof is provided in Appendix \ref{app:proof-linear-main}. \qed

\begin{remark}[Identifiability with unknown boundary]
Because the single-reward \btmodel{} distribution depends on $b$ and $V$ only through their product $bV$, the mean preference vector $\theta^\star$ is \emph{not} identifiable from choices alone when $b$ is unknown---even asymptotically. However, as \Cref{thm:linear-main} shows, response times restore identifiability: the boundary and the utility scale are both pinned down by the introduced estimators.
\end{remark}


\subsection{Recovering the utilitarian aggregator}
\label{sec:util-aggregation}

Consider a utilitarian planner who wants to choose the alternative (identified by its feature vector in $\R^d$) that maximizes society's average utility, given by\footnote{For any distribution $\kappa \in \Delta(\mathcal{X})$, we similarly define $\mathrm{AvgUtil}(\kappa):= \E_{x \sim \kappa}[\mathrm{AvgUtil}(x)]$.}
\begin{equation}
\mathrm{AvgUtil}(x) = \E_{\theta \sim G^*}[x^\top \theta] = x^\top \theta^\star.
\end{equation}

Assume $\sup_{x \in \mathcal{X}} \mathrm{AvgUtil}(x)>0$. The planner has access to a dataset $\mathcal{D}_n$ consisting of $n$ choice observations, possibly augmented with response-time observations, from $n$ users. The planner then applies a possibly randomized algorithm $\mathcal{A}$ to $\mathcal{D}_n$. The expected average utility of $\mathcal{A}$, denoted by $\mathrm{AvgUtil}_{n}(\mathcal{A})$, is defined as
\begin{equation*}
\mathrm{AvgUtil}_{n}(\mathcal{A})
:=
\E\!\left[\mathrm{AvgUtil}\bigl(\mathcal{A}(\mathcal{D}_n)\bigr)\right],    
\end{equation*}
where the expectation is taken over the randomness in the dataset $\mathcal{D}_n$ and any additional randomness in the algorithm.

The distortion of $\mathcal{A}$ is the worst-case asymptotic competitive ratio between the optimal average utility and the average utility achieved by the algorithm. This is formally defined as:
\begin{equation}
\mathrm{dist}(\mathcal{A})
:=
\sup_G
\limsup_{n \to \infty}
\frac{\sup_{x \in \mathcal{X}} \mathrm{AvgUtil}(x)}{\mathrm{AvgUtil}_{n}(\mathcal{A})}.
\end{equation}

In other words, the distortion quantifies whether there is an inherent limit on how well a utilitarian decision maker can recover the optimal decision from preference data, even as the dataset size $\to \infty$.

Let us first consider the setting without response-time data. \citet[Theorem 3]{golz2025distortion} provide a lower bound of
$
\frac{\beta}{2}\cdot \frac{1+\exp(-\beta)}{1-\exp(-\beta)}
$
on the distortion rate, where $\beta$ is the temperature parameter, corresponding to the boundary $b$ in our model up to a constant factor of $2$. This lower bound can also be realized by a linear-utility model in higher dimensions. Thus, even with infinitely many choice observations, a utilitarian planner may fail to recover an alternative with optimal average utility.

On the other hand, suppose that response-time data are also included in the dataset $\mathcal{D}$. A natural algorithm $\mathcal{A}$ in this setting is to select
$
\arg\max_{x\in \cX}\, x^\top \thetahat.
$
Assume that the alternatives are uniformly bounded, and that the above maximizer exists. Then \Cref{thm:linear-main} implies that, for every $\varepsilon>0$, the probability that the gap between the maximum average utility and the average utility of the algorithm's output exceeds $\varepsilon$ converges to zero. Consequently, the distortion rate in this setting is equal to one. In this sense, response-time data allow the utilitarian planner to bypass the choice-only distortion barrier and asymptotically recover an optimal decision.
\section{Experiments}
\label{sec:experiments}

We validate our estimators from Sections~\ref{sec:estimators} and \ref{sec:linear} on synthetic \ddm{} data and on an intertemporal-choice dataset, collected by \citet{amasino2019amount}. In every experiment the boundary is treated as \emph{unknown} and estimated from response times alone via the bias-corrected estimator $\btil$ of Eq.~\eqref{eq:btil-def} with $\lambda_n=(\log n)^{3/2}$; the weight $w_\beta$ is truncated at $K=100$ terms; and each reported value averages over $R=50$ independent Monte-Carlo replications. Full experimental details appear in Appendix~\ref{app:experiments}.

\subsection{Synthetic experiments}
\label{sec:exp-synthetic}

\begin{figure}[t]
  \centering
  \begin{subfigure}[t]{0.47\linewidth}
    \centering
    \includegraphics[width=\linewidth]{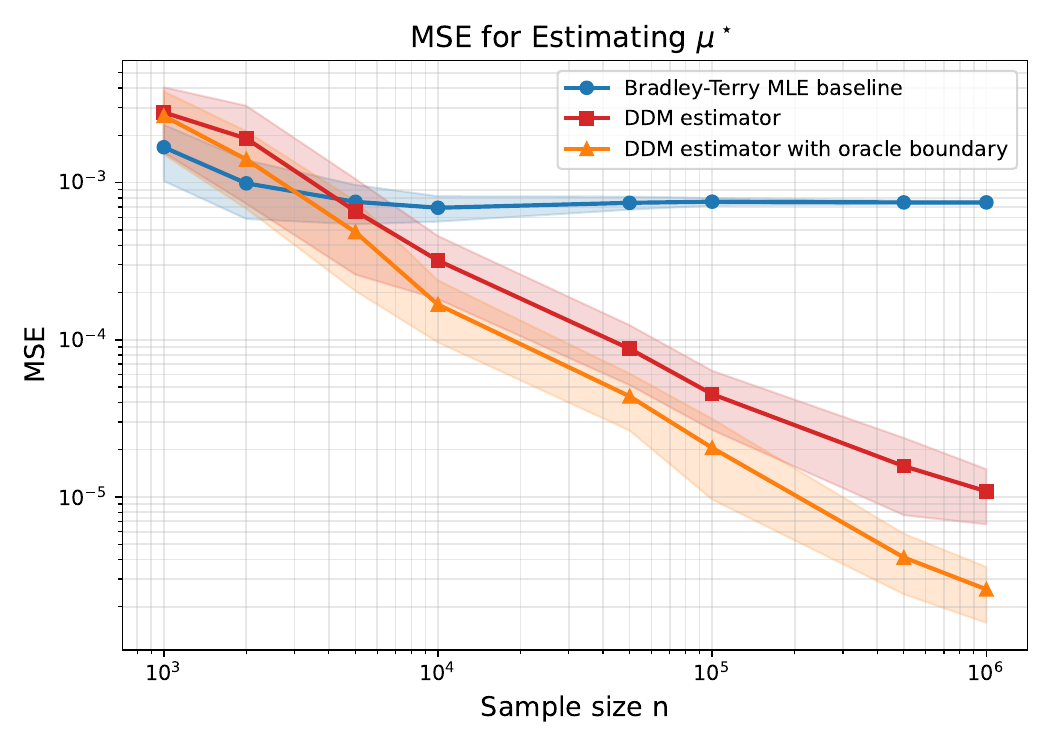}
    \caption{Uniform prior on $V$.}
    \label{fig:exp_basic_uniform}
  \end{subfigure}
  \hfill
  \begin{subfigure}[t]{0.47\linewidth}
    \centering
    \includegraphics[width=\linewidth]{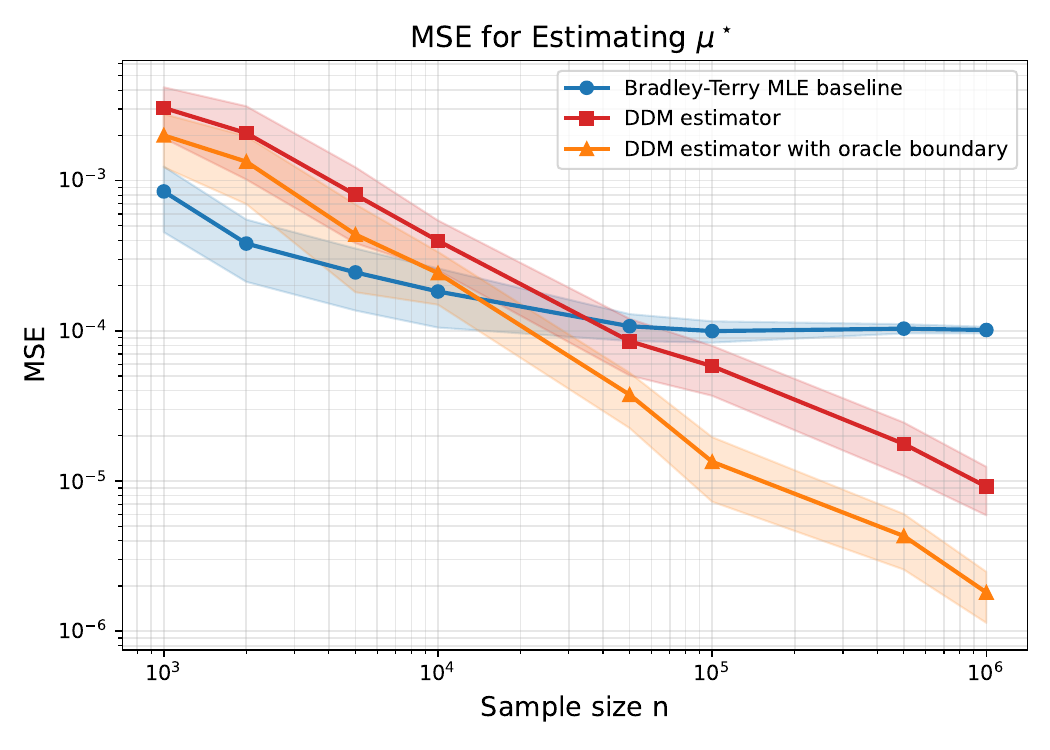}
    \caption{Beta prior on $V$.}
    \label{fig:exp_basic_beta}
  \end{subfigure}
  \caption{\textbf{Tabular setting.} Monte-Carlo MSE of $\muhat$ versus sample size $n$ for two latent-drift priors. The DDM estimator dominates the baseline, which plateaus at a heterogeneity-induced bias floor. Shaded bands denote pointwise $95\%$ confidence intervals over $R=50$ replications.}
  \label{fig:exp_basic_mse}
\end{figure}

\paragraph{Data generating process.}
We study two regimes mirroring Sections~\ref{sec:estimators} and~\ref{sec:linear}. In the \emph{tabular} setting, we fix two alternatives and sample drifts $v_1,\dots,v_n$ i.i.d.\ from $\Fstar$ with mean $\mu^\star=0.25$ and boundary $b^\star=1.25$; we consider both Uniform and Beta priors on the drift to reflect heterogeneous preference. In the \emph{linear} setting, we consider the linear drift model $v_i=\psi_i^\top\theta_i$ with unknown boundary $b=1.25$. Each alternative vector  $\psi_i\sim\mathrm{Uniform}[-1,1]^4$, and the mean preference feature is $\theta^\star=(0.25,-0.15,0.10,-0.30)$, where each coordinate of $\theta_i$ follows one of four priors with variance $\sigma_\theta^2=0.25$: Gaussian, Uniform, Beta, or Laplace (truncated to $\pm 6\sigma_\theta$).

\paragraph{Methods.}
In both regimes we compare the following three estimators.
\begin{itemize}[leftmargin=1em,itemsep=1pt,topsep=0.5pt]
    \item \emph{\btmodel{} MLE baseline}: the MLE minimizer of \btmodel{} (temperature calibrated by the oracle boundary $b^\star$) over the choice-only data. In the \emph{tabular} setting, $\widehat\mu_n^{\mathrm{BT}}=\mathrm{arctanh}\bigl(\tfrac{1}{n}\sum_i Z_i\bigr)/ b^\star$; in the \emph{linear} setting, the pooled \btmodel{} logistic MLE $\widehat\theta_n^{\mathrm{BT}}=\arg\min_\theta \frac{1}{n}\sum_i\log(1+e^{-2b^\star Z_i\psi_i^\top\theta})$.
    \item \emph{Drift Diffusion Model (DDM)}: the consistent estimator developed by \Cref{thm:plugin-scalar} and \Cref{thm:linear-main}. In the \emph{tabular} setting, $\hat \mu_n^{\mathrm{DDM}} =\frac{1}{n}\sum_i Z_iw_{\btil}(\tau_i)$; in the \emph{linear} setting, $\hat \theta_n^{\mathrm{DDM}}$ is the OLS estimator in Eq.~\eqref{eq:thetahat}. Both use the bias-corrected boundary estimate $\btil$.
    \item \emph{DDM with Oracle Boundary:} identical to the previous Drift Diffusion Model estimator except replaces $\btil$ by the oracle boundary $b^\star$.
\end{itemize}

\paragraph{Results.}
We measure the Monte-Carlo MSE to the ground-truth population target: $\E[(\muhat-\mu^\star)^2]$ in the tabular case and $\E[\|\thetahat-\theta^\star\|_2^2]$ in the linear case. Figures~\ref{fig:exp_basic_mse} and~\ref{fig:exp_linear_mse} show two robust patterns. First, the \btmodel{} baseline plateaus at a nonzero bias floor. In comparison, the error of our DDM estimator decays as sample size grows, and achieves substantially smaller errors. Second, using the estimated boundary $\btil$ incurs only moderate degradation
relative to the oracle-boundary estimator, supporting the practical effectiveness
of the Richardson correction.

\begin{figure}[t]
  \centering
  \begin{subfigure}[t]{0.47\linewidth}
    \centering
    \includegraphics[width=\linewidth]{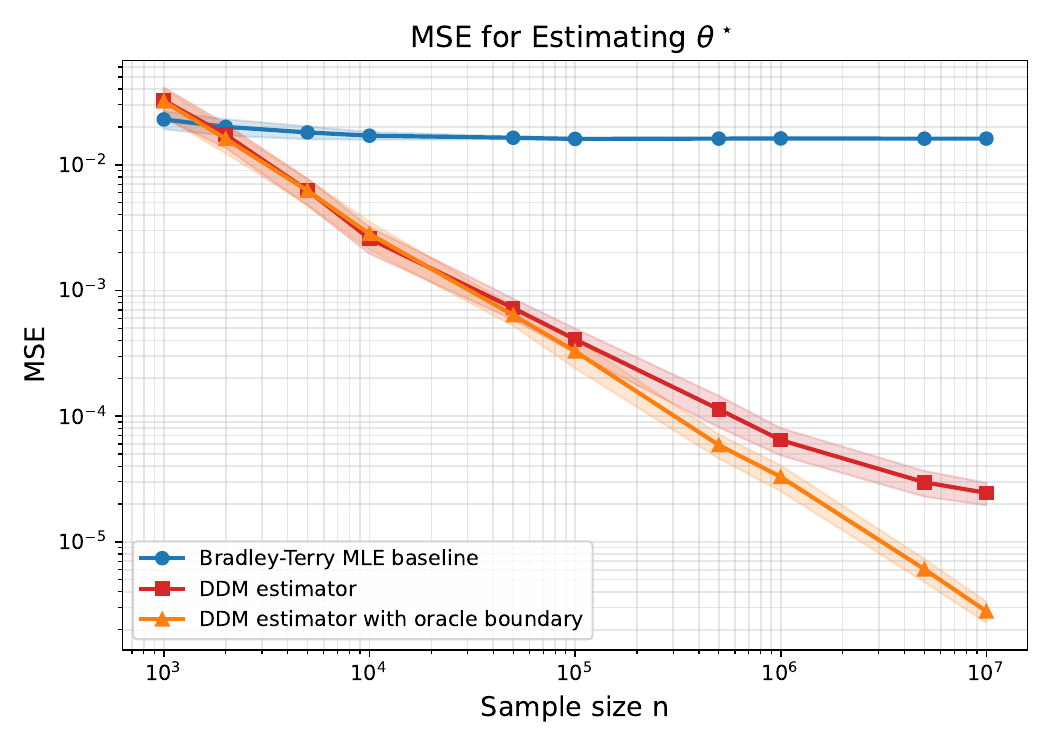}
    \caption{Gaussian.}
    \label{fig:exp_linear_gaussian}
  \end{subfigure}
  \hfill
  \begin{subfigure}[t]{0.47\linewidth}
    \centering
    \includegraphics[width=\linewidth]{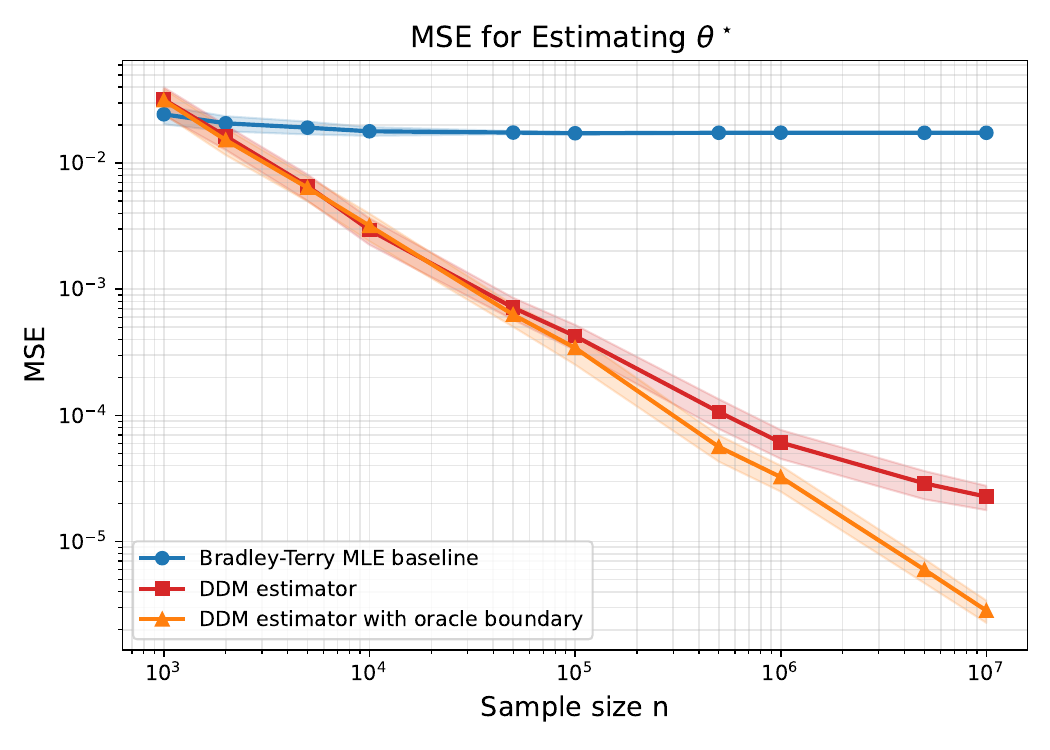}
    \caption{Uniform.}
    \label{fig:exp_linear_uniform}
  \end{subfigure}
  \begin{subfigure}[t]{0.47\linewidth}
    \centering
    \includegraphics[width=\linewidth]{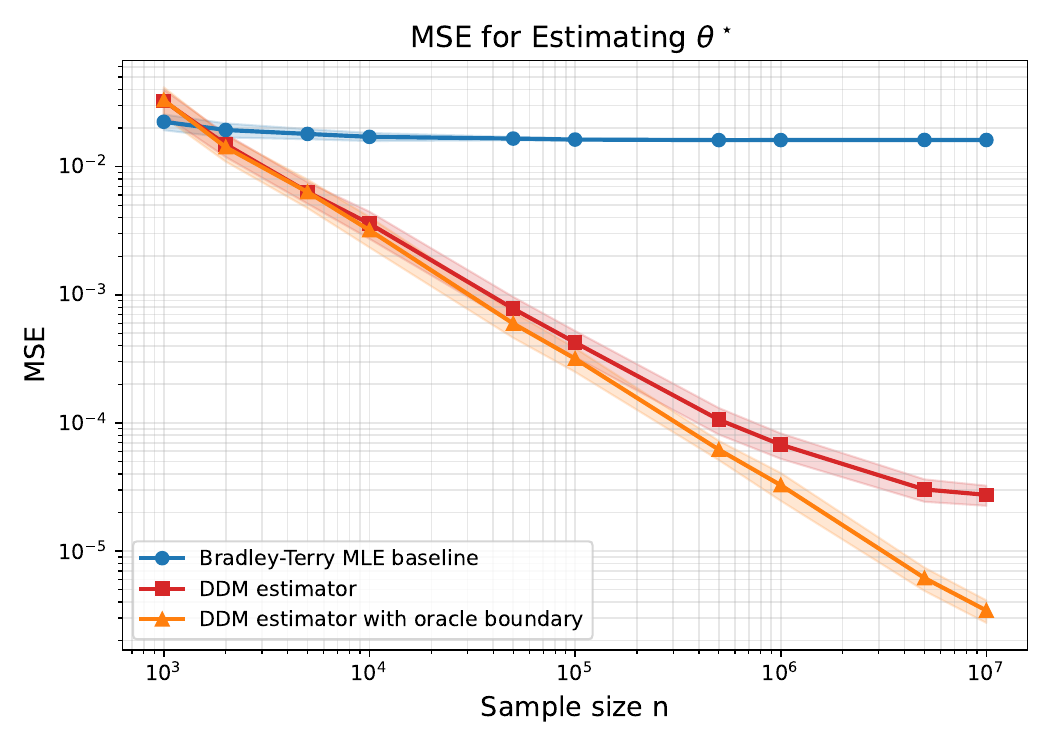}
    \caption{Beta.}
    \label{fig:exp_linear_beta}
  \end{subfigure}
  \hfill
  \begin{subfigure}[t]{0.47\linewidth}
    \centering
    \includegraphics[width=\linewidth]{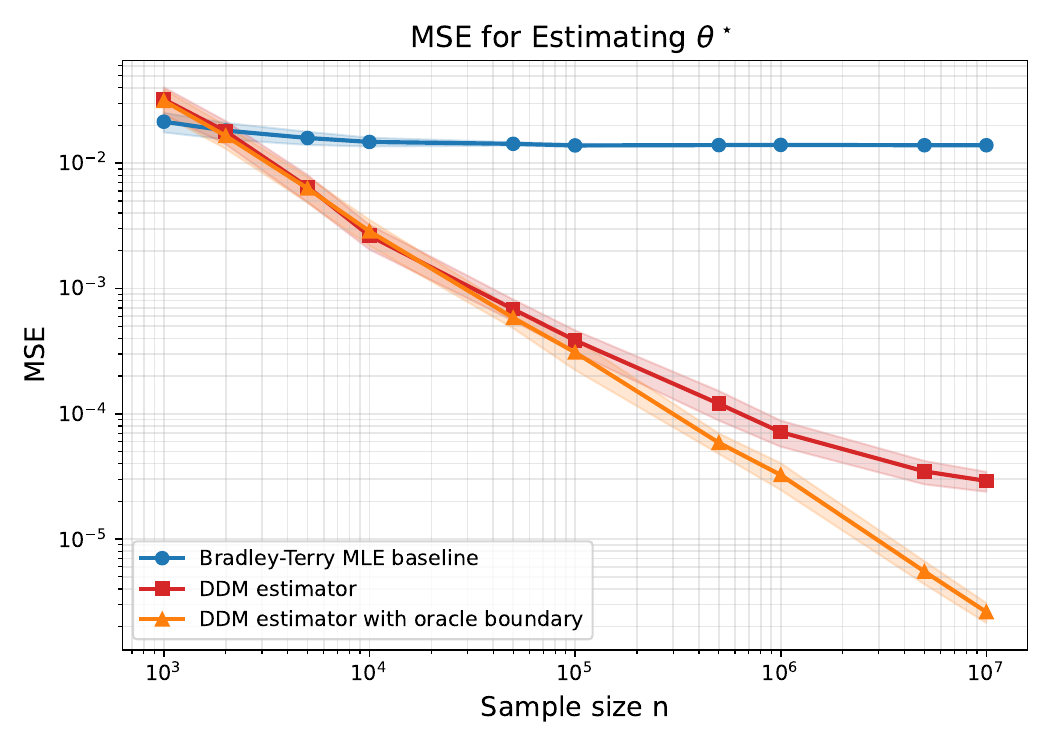}
    \caption{Laplace.}
    \label{fig:exp_linear_laplace}
  \end{subfigure}
  \caption{\textbf{Linear contextual setting.} Monte-Carlo MSE of $\thetahat$ versus sample size $n$ for four priors on $\theta_i\in\R^4$. The DDM estimator again tracks the oracle and uniformly dominates the baseline. Shaded bands denote pointwise $95\%$ confidence intervals over $R=50$ replications.}
  \label{fig:exp_linear_mse}
\end{figure}

\subsection{Real-data experiment}
\label{sec:exp-real}

\paragraph{Setup.} We use the publicly available intertemporal-choice dataset of \citet{amasino2019amount}. Each trial asks a participant to choose between a smaller-sooner reward $s_r\in\{0.5,\dots,9.5\}$ paid today and a larger-later reward $\ell_r=10$ paid after delay $\ell_d\in\{1,7,15,30,90,180,365\}$ days, with response time $\tau_i$ recorded. Following \citet{echenique2025response}, we encode the trial by $\psi_i=\big((\ell_r-s_r)/9.5,\,-\ell_d/365\big)\in\R^2$. After dropping missing trials and two participants with only one choice class, the pooled sample contains $n_{\max}=13{,}793$ observations from $S=98$ participants. The target is the subject-level average preference feature $\theta^\star = \frac{1}{S}\sum_{s=1}^S\widehat\theta^{(s)} \in \R^2$, computed by averaging per-participant \btmodel{} MLE $\widehat\theta^{(s)}$. Then we subsample $n$ anonymous data points used to estimate $\theta^\star$.

\begin{figure}[t]
  \centering
  \includegraphics[width=0.7\linewidth]{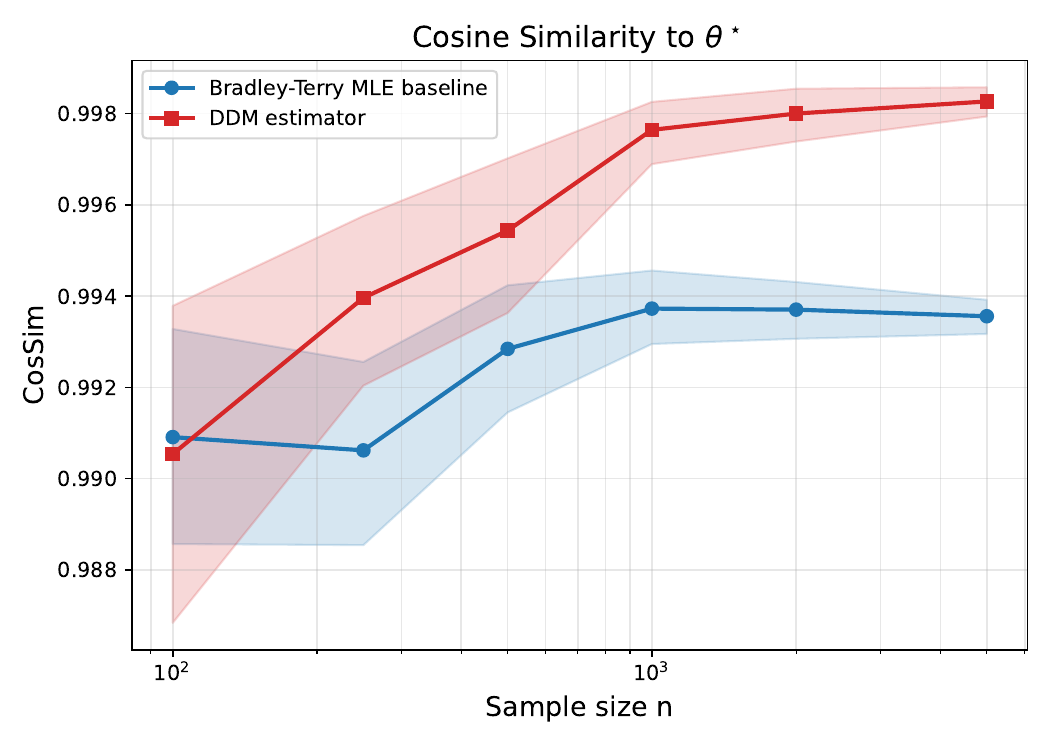}
  \caption{\textbf{Intertemporal choice data.} Cosine similarity between each estimator and the subject-level target versus subsample size $n$ on \citet{amasino2019amount} dataset. Shaded regions are 95\% bootstrap bands over $R=50$ replications. Our DDM estimator achieves higher accuracy than the baseline.}
  \label{fig:exp_real_cossim}
\end{figure}

\paragraph{Methods.} We compare the following two estimators.
\begin{itemize}[leftmargin=1em,itemsep=1pt,topsep=0.5pt]
    \item \emph{\btmodel{} MLE baseline}: the \btmodel{} logistic MLE over choice-only data.
    \item \emph{Drift Diffusion Model (DDM)}: the OLS estimator in Eq.~\eqref{eq:thetahat} using the bias-corrected boundary estimate $\btil$.
\end{itemize}

\paragraph{Results.} Because the common boundary is unknown, only the \emph{direction} of $\theta^\star$ is comparable across estimators. We therefore report $\mathrm{CosSim}(\widehat\theta,\theta^\star)$, averaged over $R=50$ subsamples of size $n$. Figure~\ref{fig:exp_real_cossim} shows that the DDM estimator is closer to the target than the \btmodel{} MLE baseline at nearly every subsample sizes. As $n$ grows, the baseline plateaus around $\mathrm{CosSim} \approx 0.993-0.994$, whereas our method reaches about $0.998$ at the largest evaluated subsample size $n=5000$. This is consistent with the theory: response times retain preference information that choices alone discard, even when the \ddm{} is only an approximation.

\section{Conclusion}
\label{sec:conclusion}
Under a drift-diffusion model with a common but unknown decision boundary, binary choices combined with response times are sufficient to construct a consistent estimator of the population-average preference in a heterogeneous anonymous labeler pool, a target that is not identifiable from choices alone.

\paragraph{Limitations and future work.} The theory assumes that choices are well-described by a \ddm{} with a common boundary; in practice, caution may vary across annotators or over time, and the model can be misspecified in other ways as well. The relaxation of these assumptions is discussed in the appendix. 
Moreover, the boundary estimator relies on a large-$\lambda$ Laplace tail and converges more slowly in finite samples than the drift estimator. Extending the analysis to heterogeneous boundaries and integrating the response-time pseudo-outcomes into end-to-end alignment objectives are natural next steps.

\section{Acknowledgment}
Part of this work was supported by the Simons Institute for the Theory of Computing, and conducted when Alireza Fallah visited the Institute. We also wish to acknowledge funding by the European Union (ERC-2022-SYG-OCEAN-101071601). Views and opinions expressed are however those of the author(s) only and do not
necessarily reflect those of the European Union or the European Research Council Executive Agency. Neither the European Union nor the granting authority can be held responsible for them.

\bibliography{ref}

\newpage
\appendix

\section{Omitted Proofs}
\label{app:proofs}

This appendix collects the full proofs of the results stated in the main body.
\subsection{Proof of \Cref{thm:unbiased}}
\label{app:proof-unbiased}

Fix $b>0$. We first introduce some notation. Let $p_v(z,t)$ denote the joint density of $(Z,T)$ under drift $V=v$. Let $\Prob_v$ denote the marginal distribution of $Z$. Finally, $\E_v$ denotes the expectation under a \ddm{} with fixed drift $V=v$ and boundary $b$. The proof has three stages.

\paragraph{Stage 1: Joint distribution of $Z$ and $T$.}
Applying Girsanov’s theorem, we obtain the tilting identity (see also \citep{echenique2025response}[Lemma 3])
\begin{align*}
p_v(z,t) \;=\; p_0(z,t)\,\exp\!\big(bvz-\frac{1}{2}v^2 t \big).
\end{align*}
Under zero drift and symmetric boundaries, $Z$ is symmetric and independent of $T$, so $p_0(z,t)=\frac12 f_0(t)$ for some density $f_0$. Hence
\begin{align}\label{eq:app-tilt}
    p_v(z,t) \;=\; \frac12 f_0(t)\,\exp\!\big(bvz-\frac{1}{2}v^2 t\big).
\end{align}
Summing \eqref{eq:app-tilt} over $z\in\{\pm 1\}$ gives the marginal of $T$ under drift $v$, and the conditional probability of $Z$ given $T$ is
\begin{align*}
    \Prob_v(Z=+1 \mid T=t) \;=\; \frac{e^{bv}}{e^{bv}+e^{-bv}} \;=\; \frac{1}{1+e^{-2bv}},
\end{align*}
which does not depend on $t$. Therefore $Z\perp\!\!\!\perp T$ under drift $v$, with $\E_v[Z]=\tanh(bv)$.

\paragraph{Stage 2: Reduction to a Laplace identity.}
Because $Z$ and $T$ are independent under drift $v$,
\begin{align}\label{eq:app-indep}
    \E_v[Z w_b(T)] \;=\; \tanh(bv)\,\E_v[w_b(T)].
\end{align}
Therefore, to establish $\E_v[Zw_b(T)]=v$ for all $v$, it suffices to show
\begin{align}\label{eq:app-goal}
    \E_v[w_b(T)] \;=\; \frac{v}{\tanh(bv)}.
\end{align}
Using \eqref{eq:app-tilt},
\begin{align*}
    \E_v[w_b(T)]
    &= \int_0^\infty w_b(t) f_0(t)\,\frac{e^{bv}+e^{-bv}}{2}\,e^{-v^2t/2}\,dt
    = \cosh(bv)\,G(v^2/2),
\end{align*}
where $G(s):=\int_0^\infty w_b(t)f_0(t)e^{-st}\,dt$. Note that $G(s)$ is the Laplace transform of $w_b(t)f_0(t)$.
Matching to \eqref{eq:app-goal} and canceling $\cosh(bv)$, we can see that we have to show
\begin{align}\label{eq:app-G}
    G(s) \;=\; \frac{\sqrt{2s}}{\sinh(b\sqrt{2s})}.
\end{align}
In other words, we have to show that $w_b(t)f_0(t)$ is the inverse Laplace transform of the right-hand side.

\paragraph{Stage 3: Series inversion.}
First note that the distribution of response time under zero drift, i.e., $f_0(t)$, is given by \citep{karatzas1991brownian}[Equation 8.24]
\begin{align}\label{eq:app-f0}
    f_0(t) \;=\; \sqrt{\frac{2}{\pi}}\,t^{-3/2}\sum_{k=0}^\infty (-1)^k \ck(b)\exp\!\left(-\frac{\ck(b)^2}{2t}\right).
\end{align}

Next, note that, by expanding $1/\sinh(y)=2e^{-y}/(1-e^{-2y})=2\sum_{k\ge0}e^{-(2k+1)y}$ and multiplying by $\sqrt{2s}$, we have
\begin{align*}
    \frac{\sqrt{2s}}{\sinh(b\sqrt{2s})} \;=\; 2\sum_{k=0}^\infty \sqrt{2s}\,e^{-\ck(b)\sqrt{2s}}.
\end{align*}
Now, to compute the inverse Laplace transform of the left hand side, first note that the inverse Laplace transform is a linear operator, and so, it suffices to find the inverse Laplace transform of $\mathcal{L}^{-1}[\sqrt{2s}\,e^{-c\sqrt{2s}}]$ for every $c$.

Notice that
\begin{equation}
\mathcal{L}^{-1}\!\left[e^{-c\sqrt{s}}\right](t)
=
\frac{c}{2\sqrt{\pi}\,t^{3/2}}
\exp\!\left(-\frac{c^{2}}{4t}\right),
\qquad c>0,\ t>0.    
\end{equation}
Next,
differentiating $\mathcal{L}^{-1}[e^{-c\sqrt{2s}}](t)$ with respect to $c$ gives
$-\mathcal{L}^{-1}[\sqrt{2s}\,e^{-c\sqrt{2s}}](t)=\frac{1}{\sqrt{2\pi}}t^{-3/2}e^{-c^2/(2t)}(\frac{c^2}{t}-1)$. 

Therefore, by linearity of the inverse Laplace transform, we have
\begin{align}\label{eq:app-q}    q(t):=\mathcal{L}^{-1}\!\left[\frac{\sqrt{2s}}{\sinh(b\sqrt{2s})}\right]\!(t)
\;=\; \sqrt{\frac{2}{\pi}}\,t^{-3/2}\sum_{k=0}^\infty\exp\!\left(-\frac{\ck(b)^2}{2t}\right)\!\left(\frac{\ck(b)^2}{t}-1\right).
\end{align}

Recall that our goal was to show the $q(t)$ above is equal to $f_0(t) w_b(t)$. Dividing \eqref{eq:app-q} by \eqref{eq:app-f0} and canceling the common prefactor $\sqrt{2/\pi}\,t^{-3/2}$ implies that we have to show $w_b(t)$ is equal to
\begin{align*}
\frac{q(t)}{f_0(t)}
\;=\; \frac{\sum_{k\ge0}e^{-\ck(b)^2/(2t)}(\ck(b)^2/t-1)}{\sum_{k\ge0}(-1)^k\ck(b)e^{-\ck(b)^2/(2t)}},
\end{align*}
which is exactly \eqref{eq:wb-def}. This gives us the desired result.

\paragraph{Stage 4: Unbiasedness and minimality.}
By \eqref{eq:app-indep} and \eqref{eq:app-goal}, $\E_v[Zw_b(T)]=v$ for every $v$. Iterating expectation under the mixture $V\sim\Fstar$ gives $\E[Zw_b(T)]=\E[V]=\mu$, so $\E[\muhat]=\mu$. 

Finally, uniqueness of $w_b$ follows from the injectivity of the Laplace transform: if another smooth $\tilde w$ satisfied \eqref{eq:app-goal}, then $\tilde w\cdot f_0$ and $w_b\cdot f_0$ would share the transform in \eqref{eq:app-G}, and injectivity would force $\tilde w=w_b$ Lebesgue-almost everywhere, hence everywhere by continuity. \qed


\subsection{Proof of \Cref{thm:one-scale}}
\label{app:proof-one-scale}

With a slight abuse of notation, let $\Lb(\lambda):=\E[e^{-\lambda T}]$, where the expectation is taken over the randomness of both $V$ and $T$.

\paragraph{Step 1: Population tail rate.} Fix $b>0$ and $\lambda\ge 0$. Under drift $v$, as stated in \eqref{eqn:laplace_transform_drift}, $\E[e^{-\lambda T}\mid V=v]=\cosh(bv)/\cosh(b\sqrt{v^2+2\lambda})$; therefore, we have:
\begin{align}\label{eq:app-Lb}
    \Lb(\lambda) \;=\; \E_{V\sim\Fstar}\!\left[\frac{\cosh(bV)}{\cosh(b\sqrt{V^2+2\lambda})}\right].
\end{align}
For $v\in[-M,M]$, $\cosh(bv)\ge 1$ and $\cosh(x)\le e^x$, so $\cosh(bv)/\cosh(b\sqrt{v^2+2\lambda})\ge e^{-b\sqrt{M^2+2\lambda}}$, yielding
\begin{align}\label{eq:app-Lb-lower}
    \Lb(\lambda) \;\ge\; e^{-b\sqrt{M^2+2\lambda}}.
\end{align}
Conversely, $\cosh(bv)\le e^{bM}$ and $\cosh(x)\ge e^x/2$, so $\cosh(bv)/\cosh(b\sqrt{v^2+2\lambda})\le 2e^{bM}e^{-b\sqrt{2\lambda}}$, giving
\begin{align}\label{eq:app-Lb-upper}
    \Lb(\lambda) \;\le\; 2e^{bM}e^{-b\sqrt{2\lambda}}.
\end{align}
Combining \eqref{eq:app-Lb-lower} and \eqref{eq:app-Lb-upper} implies
\begin{align*}
    -\frac{b\sqrt{M^2+2\lambda}}{\sqrt{2\lambda}} \;\le\; \frac{\log\Lb(\lambda)}{\sqrt{2\lambda}} \;\le\; \frac{\log 2+bM}{\sqrt{2\lambda}}-b.
\end{align*}
Both sides converge to $-b$ as $\lambda\to\infty$, and hence, we have $-\log\Lb(\lambda)/\sqrt{2\lambda}\to b$.

\paragraph{Step 2: From the empirical distribution to the population distribution.} Note that $X_{n,i}:=e^{-\lambda_n t_i}\in[0,1]$ are i.i.d.\ with mean $\Lb(\lambda_n)$. We also denote the variance of $X_{n,i}$ by $\Var(X_{n,i})$.
We next make the following claim.
\begin{lemma} \label{lemma:nVar_infty}
Suppose the conditions in \eqref{eq:lambda-rate} hold. Then, $\lim_{n \to \infty} n\Var(X_{n,i}) = \infty$.
\end{lemma}
\begin{proof}
First, by the law of total variance, it suffices to show this condition on any fixed drift $V=v$, given that
\begin{equation*}
\Var(X_{n,i}) \geq \mathbb{E}_V \left [\Var(X_{n,i} ~ |~ V=v)\right ].    
\end{equation*}
Next, note that we have
\begin{align}
 \Var(X_{n,i} ~ |~ V=v) & =  \E[e^{-2\lambda_n T}\mid V=v]  - (\E[e^{-\lambda_n T}\mid V=v])^2 \nonumber \\
 & \geq  \exp(-b\sqrt{M^2+4\lambda_n}) - 4 \exp(2bM-2b\sqrt{2\lambda_n}), \label{eqn:variance_lb}
\end{align}
where the last inequality follows from the same argument that we used to derive \eqref{eq:app-Lb-lower} and \eqref{eq:app-Lb-upper}. Now, looking at \eqref{eqn:variance_lb}, we note that, as $\lambda_n \to \infty$, ignoring the constants, the first term would decay as $\exp(-2b\sqrt{\lambda_n})$, while the second term would decay as $\exp(-2b \sqrt{2\lambda_n})$. Therefore, there exists $n^*$ such that, for any $n \geq n^*$, we have
\begin{equation}
\exp(-b\sqrt{M^2+4\lambda_n}) - 4 \exp(2bM-2b\sqrt{2\lambda_n}) \geq \zeta \exp(-2b\sqrt{\lambda_n}),    
\end{equation}
for some constant $\zeta >0$. As a result, for any $n \geq n^*$, we have
\begin{equation}
n \Var(X_{n,i}) \geq \zeta n \exp(-2b\sqrt{\lambda_n}).    
\end{equation}
The condition $\sqrt{\lambda_n}=o(\log n)$ completes the proof.
\end{proof}
Now, let us define $\tilde{X}_{n,i} = X_{n,i} - L_b(\lambda_n)$. We also define $s_n^2 := n\Var(\tilde{X}_{n,i}) = n\Var(X_{n,i})$ and $S_n := \sum_{i=1}^n \tilde{X}_{n,i}$. Note that, $X_{n,i}\le 1$, and hence, $|\tilde{X}_{n,i}| \leq 1$.

The Lindeberg-Feller Central Limit Theorem \citep{billingsley2017probability}[Theorem 27.2] states that, if the Lindeberg condition, given as
\begin{equation} \label{eqn:Lindeberg_condition}
\lim_{n \to \infty} \frac{1}{s_n^2} \sum_{i=1}^n \mathbb{E}\left[ \tilde{X}_{n,i}^2; |\tilde{X}_{n,i}| \geq \varepsilon s_n \right] = 0 \quad \text{ for all } \varepsilon >0.   
\end{equation}
holds, then, $S_n/s_n$ converges to $\mathcal{N}(0,1)$ in distribution. 

Notice that the Lindeberg condition \eqref{eqn:Lindeberg_condition} indeed holds, as $\tilde{X}_{n,i}$ is bounded and $s_n \to \infty$ by the above lemma. 
Therefore, we obtain that
\begin{equation}
\frac{S_n}{s_n} = 
\sqrt{n} \cdot  \frac{\Lhat(\lambda_n)-\Lb(\lambda_n)}{\sqrt{\Var(X_{n,i})}}     
\end{equation}
converges to $\mathcal{N}(0,1)$ in distribution. This implies that, for any $\tilde{M}$, we have 
\begin{equation}
\lim_{n \to \infty} \text{Pr} \left( \left | \sqrt{n} \cdot  \frac{\Lhat(\lambda_n)-\Lb(\lambda_n)}{\sqrt{\Var(X_{n,i})}} \right | \geq \tilde{M} \right) = 2(1-\Phi(\tilde{M})),    
\end{equation}
where $\Phi(\cdot)$ is the CDF of the standard normal distribution. 
Consequently, we have
\begin{equation} \label{eqn:conv_prob_1}
\lim_{n \to \infty} \text{Pr} \left( \left | \frac{\Lhat(\lambda_n)}{\Lb(\lambda_n)} - 1 \right | \geq \frac{\tilde{M} \sqrt{\Var(X_{n,i})}}{\sqrt{n}\Lb(\lambda_n)} \right) 
= 2(1-\Phi(\tilde{M})).
\end{equation}
Note that $\Var(X_{n,i}) \leq \Lb(\lambda_n)$, because
\begin{equation} \label{eqn:var_Lb_bound}
\Var(X_{n,i})\le\E[X_{n,i}^2]\le\E[X_{n,i}]=\Lb(\lambda_n),    
\end{equation}
and thus, we deduce from \eqref{eqn:conv_prob_1} that
\begin{equation} \label{eqn:conv_prob_2}
\limsup_{n \to \infty} \text{Pr} \left( \left | \frac{\Lhat(\lambda_n)}{\Lb(\lambda_n)} - 1 \right | \geq \frac{\tilde{M}}{\sqrt{n \Lb(\lambda_n)}} \right) 
\leq 2(1-\Phi(\tilde{M})),
\end{equation}
for any $\tilde{M}$. Notice that, by \eqref{eqn:var_Lb_bound}, we have 
\begin{equation}
\sqrt{n \Lb(\lambda_n)} \geq \sqrt{n \Var(X_{n,i})},    
\end{equation}
and therefore, by \Cref{lemma:nVar_infty}, $\sqrt{n \Lb(\lambda_n)} \to \infty$ as $n$ grows. 
Using this fact, we claim \eqref{eqn:conv_prob_2} implies that $\Lhat(\lambda_n)/\Lb(\lambda_n)\xrightarrow{P}1$. We can prove this by contradiction. Suppose this is not the case. Then, there exists positive $\varepsilon$ and $\delta$ such that
\begin{equation}
\limsup_{n \to \infty} \text{Pr} \left( \left | \frac{\Lhat(\lambda_n)}{\Lb(\lambda_n)} - 1 \right | \geq \varepsilon \right) 
\geq \delta.
\end{equation}
However, setting $\tilde{M}$ in \eqref{eqn:conv_prob_2} such that $2(1-\Phi(\tilde{M})) < \delta$ leads to a contradiction as $\tilde{M}/\sqrt{n \Lb(\lambda_n)}$ would fall below $\varepsilon$ for large enough $n$. This completes the proof of our claim. 

Having $\Lhat(\lambda_n)/\Lb(\lambda_n)\xrightarrow{P}1$ means $\log\Lhat(\lambda_n)-\log\Lb(\lambda_n)\xrightarrow{P}0$. 
Dividing by $\sqrt{2\lambda_n}\to\infty$ and combining with Step 1 gives $\bhat=-\log\Lhat(\lambda_n)/\sqrt{2\lambda_n}\xrightarrow{P}b$. \qed

\subsection{Proof of \Cref{thm:two-scale}}
\label{app:proof-two-scale}

We establish the second-order expansion of $\log\Lb(\lambda)$ and then transfer it to the empirical transform.

\paragraph{Step 1: Second-order expansion of $\log\Lb(\lambda)$.}
Set $s:=\sqrt{2\lambda}$ and, for $v\in[-M,M]$, define $R_s(v):=\sqrt{s^2+v^2}-s=v^2/(\sqrt{s^2+v^2}+s)$. 

For $x\ge 0$, 
\begin{equation}
\frac{1}{\cosh(x)}=2e^{-x}(1+\eta(x)),    
\end{equation}
with $|\eta(x)|\le e^{-2x}$. Applying this with $x=b\sqrt{s^2+v^2}=b(s+R_s(v))$, and using $\sqrt{s^2+v^2}\ge s$, we obtain
\begin{align*}
    \frac{1}{\cosh(b\sqrt{s^2+v^2})} \;=\; 2e^{-bs}e^{-bR_s(v)}(1+\eta),\qquad |\eta|\le e^{-2bs}.
\end{align*}
Multiplying by $\cosh(bv)$ and taking expectation over $V\sim\Fstar$ yields
\begin{align*}
    \Lb(\lambda) \;=\; 2e^{-bs}\,\E[\cosh(bV)e^{-bR_s(V)}] + O(e^{-3bs}).
\end{align*}

Note that, there exist constants $K_1=K_1(M)$, $K_2=K_2(M)$ such that, uniformly on $|v|\le M$ as $s\to\infty$,
\begin{align*}
    \sup_{|v|\le M}\bigl|R_s(v)-\frac{v^2}{2s}\bigr|\le\frac{K_1}{s^3},
    \qquad
    \sup_{|v|\le M}|R_s(v)|\le\frac{K_2}{s}.
\end{align*}
Therefore, the following expansion holds uniformly on $|V|\le M$
\begin{equation}
e^{-bR_s(V)}=1-bR_s(V)+O(R_s(V)^2)=1-\frac{bV^2}{2s}+O(s^{-2}),    
\end{equation}
which implies
\begin{align}
    \E[\cosh(bV)e^{-bR_s(V)}] \;=\; C_0 - \frac{bC_2}{2s} + O(s^{-2}),
\end{align}
with $C_0=\E[\cosh(bV)]\ge 1$ and $C_2=\E[V^2\cosh(bV)]$. Absorbing the $O(e^{-3bs})$ remainder into $O(s^{-2})e^{-bs}$ gives
\begin{align*}
    \Lb(\lambda) \;=\; 2C_0 e^{-bs}\bigl(1-\frac{bC_2}{2C_0 s}+O(s^{-2})\bigr).
\end{align*}
Taking logarithms and using $\log(1+u)=u+O(u^2)$, we obtain
\begin{align}
    \log\Lb(\lambda) \;=\; -bs+\log(2C_0)-\frac{bC_2}{2C_0 s}+O(s^{-2}) \;=\; -b\sqrt{2\lambda}+\log(2C_0)-\frac{bC_2}{2C_0\sqrt{2\lambda}}+O(\lambda^{-1}),
\end{align}
which proves \eqref{eq:laplace-expansion}.

\paragraph{Step 2: Cancellation at two scales.}
Applying \eqref{eq:laplace-expansion} at $\lambda$ and $4\lambda$ (so $\sqrt{2(4\lambda)}=2s$),
\begin{align*}
    \log\Lb(\lambda)-\log\Lb(4\lambda) \;=\; bs-\frac{bC_2}{4C_0 s}+O(s^{-2}),
\end{align*}
and dividing by $\sqrt{2\lambda}=s$,
\begin{align*}
    \frac{\log\Lb(\lambda)-\log\Lb(4\lambda)}{\sqrt{2\lambda}} \;=\; b-\frac{bC_2}{8C_0\lambda}+O(\lambda^{-3/2}).
\end{align*}
For comparison, $-\log\Lb(\lambda)/\sqrt{2\lambda}=b-\log(2C_0)/\sqrt{2\lambda}+O(\lambda^{-1})$, so the one-scale bias is $O(\lambda^{-1/2})$ and the Richardson-corrected bias is $O(\lambda^{-1})$.

\paragraph{Step 3: The empirical estimator.}
Note that the sequence $\{4 \lambda_n\}$ also satisfies \eqref{eq:lambda-rate}, and hence, by \Cref{thm:one-scale},  
\begin{equation}
-\frac{\log\Lhat(\lambda_n)}{\sqrt{2\lambda_n}} \xrightarrow{P}b, \quad
-\frac{\log\Lhat(4\lambda_n)}{2\sqrt{2\lambda_n}} \xrightarrow{P} b
\end{equation}
This immediately implies 
\begin{equation}
\frac{\log\Lhat(\lambda_n) - \log\Lhat(4\lambda_n)}{\sqrt{2\lambda_n}} \xrightarrow{P} b
\end{equation}
which completes the proof.
\qed


\subsection{Proof of \Cref{thm:plugin-scalar}}
\label{app:proof-plugin-scalar}

Fix a realized boundary $b>0$ and write $\Prob_b,\E_b$ for the conditional probability and expectation. The proof relies on a uniform envelope for the family $\{w_\beta\}$, stated next, and whose proof we give at the end of this subsection.

\begin{lemma}[Uniform envelope for the weight family] \label{lem:envelope}
Fix a compact interval $K=[\underline{b},\bar{b}]\subset(0,\infty)$. Then:
\begin{enumerate}
    \item For every $\tau>0$, $\beta\mapsto w_\beta(\tau)$ is continuous on $K$.
    \item There exist constants $0<\tau_0\le 1$, $\tau_1\ge 1$, $C_K<\infty$, and $c_K>0$, all depending only on $K$ such that, for all $\beta\in K$, we have
    \begin{align*}
        |w_\beta(\tau)| &\le C_K(1+\tau^{-1}) & &\text{for } 0<\tau\le \tau_0,\\
        |w_\beta(\tau)| &\le C_K & &\text{for } \tau_0\le \tau\le \tau_1,\\
        |w_\beta(\tau)| &\le C_K e^{-c_K \tau} & &\text{for } \tau\ge \tau_1.
    \end{align*}
\end{enumerate}
Consequently, $H_K(\tau):=C_K(1+\tau^{-1})\indic_{(0,\tau_0]}(\tau)+C_K\indic_{(\tau_0,\tau_1]}(\tau)+C_K e^{-c_K\tau}\indic_{[\tau_1,\infty)}(\tau)$ satisfies $\sup_{\beta\in K}|w_\beta(\tau)|\le H_K(\tau)$, and $\E_b[H_K(T)]<\infty$.
\end{lemma}

\paragraph{Main argument.}
Decompose
\begin{align*}
    \muhatpi-\mu
    &=\underbrace{\left(\frac{1}{n}\sum_{i=1}^n z_iw_b(t_i)-\mu\right)}_{A_n}
    + \underbrace{\frac{1}{n}\sum_{i=1}^n z_i\bigl(w_{\btil}(t_i)-w_b(t_i)\bigr)}_{R_n}.
\end{align*}
By Theorem~\ref{thm:unbiased}, we know that $\E_b[z_i w_b(t_i)]=\mu$. Furthermore,  \Cref{lem:envelope} gives $\E_b[|z_i w_b(t_i)|]\le\E_b [H_K(t_i)]<\infty$ for any compact $K\ni b$. Hence, by the weak law of large numbers, we have $A_n\xrightarrow{\Prob_b}0$.

For $R_n$, pick the compact set $K=[b/2,3b/2]$ and, for $0<\delta\le b/2$, define
\begin{equation}
r_\delta(\tau):=\sup_{|\beta-b|\le\delta}|w_\beta(\tau)-w_b(\tau)|.    
\end{equation}
\Cref{lem:envelope} implies that (i) $r_\delta(\tau)\downarrow 0$ pointwise as $\delta\downarrow 0$, and (ii) $r_\delta(\tau)\le 2H_K(\tau)$ where $H_K$ is integrable. Hence, by dominated convergence: 
\begin{equation}
\lim_{\delta \to 0} \E_b[r_\delta(T)] = 0 .  
\end{equation}
Fix $\varepsilon>0$ and pick $\delta\in(0,b/2]$ with $\E_b[r_\delta(T)]<\varepsilon/2$. On the event $\{|\btil-b|\le\delta\}$, we have
\begin{equation*}
|R_n|\le\frac{1}{n}\sum_i r_\delta(\tau_i).  
\end{equation*}
Therefore, 
\begin{align*}
    \Prob_b(|R_n|>\varepsilon)
    \;\le\; \Prob_b(|\btil-b|>\delta) + \Prob_b \left(~\frac{1}{n}\sum_i r_\delta(\tau_i)>\varepsilon \right).
\end{align*}
The first term vanishes by \Cref{thm:two-scale} and the second by the weak law of large numbers (the sample mean tends to $\E_b[r_\delta(\tau_i)]<\varepsilon/2$). So $R_n\xrightarrow{\Prob_b}0$, and $\muhatpi-\mu=A_n+R_n\xrightarrow{\Prob_b}0$. This completes the proof. \qed

\subsubsection{Proof of \Cref{lem:envelope}.}
\paragraph{Continuity.} Recall that $w_\beta(t)$ is equal to $\Nbeta(t)/\Dbeta(t)$, with
$$\Nbeta(t):=\sum_{k\ge0}\exp \left(\frac{-\ck(\beta)^2}{2t} \right) \left(\frac{\ck(\beta)^2}{t}-1 \right)$$ 
and
$$\Dbeta(t):=\sum_{k\ge0}(-1)^k\ck(\beta)\exp \left(\frac{-\ck(\beta)^2}{2t} \right).$$
For a fixed $t>0$, both of these series converge absolutely and uniformly on $K$ because, for $m\in\{1,2\}$, 
$$\sum_{k\ge0}(2k+1)^m e^{-(2k+1)^2{\underline{b}}^2/(2t)}<\infty.$$ 
Using the known identity $\Dbeta(t)=\sqrt{\pi/2}\,t^{3/2}f_{0,\beta}(t)$ with $f_{0,\beta}$ being the first-exit density for boundary $\beta$ and zero drift  (see Eq.~\eqref{eq:app-f0}), we have $\Dbeta(t)>0$ for every $t>0$, so $\wbeta(t)=\Nbeta(t)/\Dbeta(t)$ is continuous in $\beta$ on $K$.

\paragraph{Small-$t$ bound.} 
Factor 
$$\Dbeta(t)=\beta e^{-\beta^2/(2t)}\bigl[1+\sum_{k\ge 1}(-1)^k(2k+1)e^{-((2k+1)^2-1)\beta^2/(2t)}\bigr].$$ 
Since $((2k+1)^2-1)\beta^2/(2t)\ge 4\underline{b}^2/t$ for $k\ge 1$, we can choose $\tau_0>0$ so that 
$$\sup_{\beta\in K}\sum_{k\ge 1}(2k+1)e^{-((2k+1)^2-1)\beta^2/(2t)}\le 1/2$$
on $(0,\tau_0]$. Then $\Dbeta(t)\ge\frac12\beta e^{-\beta^2/(2t)}\ge\frac12 \underline{b}e^{-\beta^2/(2t)}$. 

Now control the numerator. Factor out $e^{-\beta^2/(2t)}$:
\[
N_\beta(t)
=
e^{-\beta^2/(2t)}
\left[
\frac{\beta^2}{t}
-1
+
\sum_{k=1}^{\infty}
e^{-((2k+1)^2-1)\beta^2/(2t)}
\left(
\frac{(2k+1)^2\beta^2}{t}
-1
\right)
\right].
\]

The first term satisfies
$
\left|
\frac{\beta^2}{t}
-1
\right|
\leq C_K(1+t^{-1})
$.
The tail satisfies
\[
\sum_{k=1}^{\infty}
e^{-((2k+1)^2-1)\beta^2/(2t)}
\left|
\frac{(2k+1)^2\beta^2}{t}
-1
\right|
\leq C_K(1+t^{-1}),
\]
because the sums
\[
\sum_{k=1}^{\infty}
(2k+1)^2
e^{-((2k+1)^2-1)b^2/(2t)}
\]
are uniformly bounded for $0<t\leq \tau_0$.

Therefore, $|\Nbeta(t)|\le C_K e^{-\beta^2/(2t)}(1+t^{-1})$, so $|\wbeta(t)|\le C_K(1+t^{-1})$.

\paragraph{Middle range.} 
Continuity of $(\beta,\tau)\mapsto\wbeta(\tau)$ on the compact $K\times[\tau_0,\tau_1]$ yields $\sup|\wbeta(\tau)|<\infty$.

\paragraph{Large-$t$ bound.} 
Note that
\begin{equation*}
w_\beta(t) = \frac{\Nbeta(t)}{\Dbeta(t)} = \frac{\sqrt{\frac{2}{\pi}}t^{-3/2}\Nbeta(t)}{f_{0,\beta}(t)}.   
\end{equation*}
We next make two claims:
\begin{lemma} \label{lemma:f_0}
Recall that $f_{0,\beta}(t)$ is given by
\begin{equation} \label{eqn:f_0_1}
f_{0,\beta}(t) = \sqrt{\frac{2}{\pi}}\,t^{-3/2}\sum_{k=0}^\infty (-1)^k \ck(\beta)\exp\!\left(-\frac{\ck(\beta)^2}{2t}\right). 
\end{equation}
This function also admits the following characterization:
\begin{equation} \label{eqn:f_0_2}
f_{0,\beta}(t) = \frac{\pi}{2\beta^2}\sum_{m\ge 0}(-1)^m(2m+1)\exp \left(-\frac{(2m+1)^2\pi^2t}{8\beta^2}\right).    
\end{equation}    
\end{lemma}
\begin{lemma} \label{lemma:q_beta}
Define $q_\beta(t)$ by  
\begin{equation} \label{eqn:q_beta_1}
q_\beta(t) := \sqrt{\frac{2}{\pi}}t^{-3/2}\Nbeta(t) = 
\sqrt{\frac{2}{\pi}}t^{-3/2} \sum_{k\ge0}\exp \left(\frac{-\ck(\beta)^2}{2t} \right) \left(\frac{\ck(\beta)^2}{t}-1 \right).
\end{equation}
This function also admits the following characterization:
\begin{equation} \label{eqn:q_beta_2}
q_\beta(t) = \frac{\pi^2}{\beta^3}\sum_{m\ge 1}(-1)^{m+1}m^2\exp \left(-\frac{m^2\pi^2 t}{2\beta^2}\right)   
\end{equation}
\end{lemma}
We defer the proof of these two claims to the end of this subsection. First let us complete the proof of \Cref{lem:envelope}, assuming these two results hold. 

Note that, for $t\ge \tau_1$ large enough, the $m=0$ term of $f_{0,\beta}$ and the $m=1$ term of $q_\beta$ dominate, giving $f_{0,\beta}(t)\ge\frac{\pi}{4\beta^2}e^{-\pi^2 t/(8\beta^2)}$ and $|q_\beta(t)|\le C_K e^{-\pi^2 t/(2\beta^2)}$. Hence
\begin{align*}
    |\wbeta(t)| \;=\; \frac{|q_\beta(t)|}{f_{0,\beta}(t)} \;\le\; C_K\exp \bigl(-\frac{3\pi^2 t}{8\beta^2}\bigr) \;\le\; C_K e^{-c_K t},
    \qquad c_K:=\frac{3\pi^2}{8\bar{b}^2}.
\end{align*}
Combining the three regions gives the envelope.
\paragraph{Integrability of $H_K$.} 
Let $g_b$ be the marginal density of $T$ under $\Prob_b$. Summing \eqref{eq:app-tilt} over $z\in\{\pm 1\}$ gives $g_b(t)=\E[\cosh(bV)e^{-V^2t/2}]f_{0,b}(t)\le\cosh(bM)f_{0,b}(t)$, and for small $t$, the formula for $f_{0,b}$ yields $f_{0,b}(t)\le Ct^{-3/2}e^{-b^2/(2t)}$. Hence
\begin{align*}
    \int_0^{\tau_0}(1+\tau^{-1})g_b(\tau)\,d\tau \;\le\; C\int_0^{\tau_0}(\tau^{-3/2}+\tau^{-5/2})e^{-b^2/(2\tau)}\,d\tau \;<\;\infty,
\end{align*}
because $e^{-b^2/(2\tau)}$ decays faster than any inverse power as $\tau \downarrow 0$. On $[\tau_0,\infty)$, $H_K$ is bounded, so its integral against $g_b$ is finite. Thus $\E_b[H_K(t_i)]<\infty$. \qed
\subsubsection{Proof of \Cref{lemma:f_0}}
First note that, the Laplace transform of each term of \eqref{eqn:f_0_1} is given by
\begin{equation}
\mathcal{L} \left[ \sqrt{\frac{2}{\pi}} t^{-3/2}~a~e^{-a^2/(2t)}\right] (\lambda) =2e^{-a\sqrt{2\lambda}}, \qquad a>0,\ \lambda>0.    
\end{equation}
This implies that, the Laplace transform of $f_{0,\beta}(t)$ in \eqref{eqn:f_0_1} is given by
\begin{equation}
\mathcal{L}[f](\lambda) = 
\sum_{k=0}^\infty (-1)^k ~ 2e^{-(2k+1)\beta\sqrt{2\lambda}}
= \operatorname{sech}(\beta\sqrt{2\lambda}).
\end{equation}
Now, the Mittag–Leffler expansion for the $\operatorname{sech}(\cdot)$ gives us
\begin{equation}
\mathcal{L}[f](\lambda) = 
\pi\sum_{m=0}^\infty~ (-1)^m~\frac{2m+1}{2\beta^2\lambda+\frac{\pi^2}{4}(2m+1)^2}
= \frac{\pi}{2\beta^2}~\sum_{m=0}^\infty~(-1)^m(2m+1)~\frac{1}{\lambda+\frac{(2m+1)^2\pi^2}{8\beta^2}}.
\end{equation}
Taking an inverse Laplace transform and using the basic fact that
\begin{equation*}
\mathcal L^{-1} \left[ \frac{1}{\lambda+\alpha}\right] (t) =e^{-\alpha t},\qquad \alpha>0, 
\end{equation*}
gives us the desired expression \eqref{eqn:f_0_2}. \qed 
\subsubsection{Proof of \Cref{lemma:q_beta}}
Note that $\Nbeta(t)$ can be cast as
\begin{equation} \label{eq:Nbeta-as-derivative}
\Nbeta(t) = -\partial_\beta\bigl(\beta S_\beta(t)\bigr)
\end{equation}
with
\begin{equation}
S_\beta(t)=\sum_{k=0}^\infty e^{-((2k+1)^2\beta^2)/(2t)}.    
\end{equation}
It is worth noting that we are allowed to use termwise differentiation given $\beta \geq \underline{b} >0$. We next compute $S_\beta(t)$ by the Jacobi theta transformation, equivalently Poisson summation. Since
\begin{align}
2S_\beta(t)
&= \sum_{n\in \mathbb{Z}}
\exp\!\left(-\frac{(2n+1)^2\beta^2}{2t}\right)
= \sum_{n\in \mathbb{Z}}
\exp\!\left(-\frac{2\beta^2}{t}\left(n+\frac12\right)^2\right),
\label{eq:double-Sbeta-shifted}
\end{align}
we apply the shifted Poisson summation formula
\begin{align}
\sum_{n\in\mathbb{Z}} f(n+\alpha)
=
\sum_{m\in\mathbb{Z}} e^{2\pi i m \alpha}\,\widehat f(m),
\label{eq:shifted-poisson}
\end{align}
with
\begin{align}
f(x) := \exp\!\left(-\frac{2\beta^2}{t}x^2\right),
\qquad
\alpha := \frac12.
\label{eq:f-and-alpha}
\end{align}
Using the Fourier transform convention
\begin{align}
\widehat f(\xi) := \int_{\mathbb{R}} f(x)e^{-2\pi i x\xi}\,dx,
\label{eq:fourier-convention}
\end{align}
the Gaussian transform is
\begin{align}
\widehat f(\xi)
=
\sqrt{\frac{\pi t}{2\beta^2}}\,
\exp\!\left(-\frac{\pi^2 t}{2\beta^2}\xi^2\right).
\label{eq:gaussian-transform}
\end{align}
Substituting \eqref{eq:f-and-alpha} and \eqref{eq:gaussian-transform} into \eqref{eq:shifted-poisson}, and then using \eqref{eq:double-Sbeta-shifted}, gives
\begin{align}
2S_\beta(t)
&=
\sqrt{\frac{\pi t}{2\beta^2}}
\sum_{m\in\mathbb{Z}} e^{2\pi i m/2}
\exp\!\left(-\frac{\pi^2 t}{2\beta^2}m^2\right)
\notag \\
&=
\sqrt{\frac{\pi t}{2\beta^2}}
\sum_{m\in\mathbb{Z}} (-1)^m
\exp\!\left(-\frac{\pi^2 t}{2\beta^2}m^2\right).
\label{eq:poisson-with-sign}
\end{align}
Since the summand is even in $m$,
\begin{align}
\sum_{m\in\mathbb{Z}} (-1)^m e^{-a m^2}
=
1 + 2\sum_{m=1}^{\infty} (-1)^m e^{-a m^2},
\qquad a>0.
\label{eq:even-sum-splitting}
\end{align}
Applying \eqref{eq:even-sum-splitting} with
\begin{align}
a = \frac{\pi^2 t}{2\beta^2},
\label{eq:a-parameter}
\end{align}
we find
\begin{align}
S_\beta(t)
&=
\frac12 \sqrt{\frac{\pi t}{2\beta^2}}
\left(
1 + 2\sum_{m=1}^{\infty} (-1)^m
\exp\!\left(-\frac{m^2\pi^2 t}{2\beta^2}\right)
\right) \notag \\
&=
\frac{\sqrt{\pi t/2}}{2\beta}
\left(
1 + 2\sum_{m=1}^{\infty} (-1)^m
\exp\!\left(-\frac{m^2\pi^2 t}{2\beta^2}\right)
\right),
\label{eq:Sbeta-after-poisson-2}
\end{align}
because $\beta>0$.

Multiplying \eqref{eq:Sbeta-after-poisson-2} by $\beta$ yields
\begin{align}
\beta S_\beta(t)
=
\frac{\sqrt{\pi t/2}}{2}
\left(
1 + 2\sum_{m=1}^{\infty} (-1)^m
\exp\!\left(-\frac{m^2\pi^2 t}{2\beta^2}\right)
\right).
\label{eq:beta-Sbeta}
\end{align}
We now differentiate \eqref{eq:beta-Sbeta} with respect to $\beta$. For each $m\geq 1$, define
\begin{align}
A_m := \frac{m^2\pi^2 t}{2}.
\label{eq:Am-definition}
\end{align}
Then
\begin{align}
\exp\!\left(-\frac{m^2\pi^2 t}{2\beta^2}\right)
=
e^{-A_m/\beta^2},
\label{eq:exp-Am}
\end{align}
and therefore
\begin{align}
\frac{d}{d\beta} e^{-A_m/\beta^2}
&=
e^{-A_m/\beta^2}\,\frac{d}{d\beta}\!\left(-A_m\beta^{-2}\right) \notag \\
&=
e^{-A_m/\beta^2}\,\left(2A_m\beta^{-3}\right) \notag \\
&=
\frac{m^2\pi^2 t}{\beta^3}
\exp\!\left(-\frac{m^2\pi^2 t}{2\beta^2}\right).
\label{eq:differentiate-exp-3}
\end{align}
Differentiating \eqref{eq:beta-Sbeta} term-by-term, we obtain
\begin{align}
\frac{\partial}{\partial \beta}\!\left(\beta S_\beta(t)\right)
&=
\frac{\sqrt{\pi t/2}}{2}
\cdot
2\sum_{m=1}^{\infty} (-1)^m
\frac{m^2\pi^2 t}{\beta^3}
\exp\!\left(-\frac{m^2\pi^2 t}{2\beta^2}\right)
 \notag \\
&=
\sqrt{\frac{\pi t}{2}}\,
\frac{\pi^2 t}{\beta^3}
\sum_{m=1}^{\infty} (-1)^m m^2
\exp\!\left(-\frac{m^2\pi^2 t}{2\beta^2}\right).
\label{eq:derivative-beta-Sbeta-2}
\end{align}
Using \eqref{eq:Nbeta-as-derivative}, we conclude that
\begin{align}
N_\beta(t)
=
\sqrt{\frac{\pi t}{2}}\,
\frac{\pi^2 t}{\beta^3}
\sum_{m=1}^{\infty} (-1)^{m+1} m^2
\exp\!\left(-\frac{m^2\pi^2 t}{2\beta^2}\right).
\label{eq:Nbeta-final-before-prefactor}
\end{align}

Finally, multiply \eqref{eq:Nbeta-final-before-prefactor} by $\sqrt{2/\pi}\,t^{-3/2}$:
\begin{align}
\sqrt{\frac{2}{\pi}}\,t^{-3/2}N_\beta(t)
&=
\sqrt{\frac{2}{\pi}}\,t^{-3/2}
\cdot
\sqrt{\frac{\pi t}{2}}
\cdot
\frac{\pi^2 t}{\beta^3}
\sum_{m=1}^{\infty} (-1)^{m+1} m^2
\exp\!\left(-\frac{m^2\pi^2 t}{2\beta^2}\right)
 \notag \\
&=
\left(
\sqrt{\frac{2}{\pi}}
\sqrt{\frac{\pi t}{2}}
t^{-3/2}
\pi^2 t
\right)
\frac{1}{\beta^3}
\sum_{m=1}^{\infty} (-1)^{m+1} m^2
\exp\!\left(-\frac{m^2\pi^2 t}{2\beta^2}\right)
 \notag \\
&=
\left(
t^{1/2} t^{-3/2} \pi^2 t
\right)
\frac{1}{\beta^3}
\sum_{m=1}^{\infty} (-1)^{m+1} m^2
\exp\!\left(-\frac{m^2\pi^2 t}{2\beta^2}\right)
 \notag \\
&=
\frac{\pi^2}{\beta^3}
\sum_{m=1}^{\infty} (-1)^{m+1} m^2
\exp\!\left(-\frac{m^2\pi^2 t}{2\beta^2}\right).
\label{eq:claimed-identity}
\end{align}
This is exactly the claimed identity \eqref{eqn:q_beta_2}. \qed


\subsection{Proof of \Cref{thm:linear-main}}
\label{app:proof-linear-main}

Throughout this proof, fix $b>0$ and define $\mathbb{P}_b$ and $\mathbb{E}_b$
as in Appendix~A.4.

\paragraph{Step 1: Boundary consistency.} Because $\|\psi_i\|\le L$ and $\|\theta_i\|\le R$, $|v_i|=|\psi_i^\top\theta_i|\le LR$ almost surely, so the scalar drifts $\{v_i\}$ satisfy the bounded-drift hypothesis of \Cref{thm:two-scale}. Hence $\btil\xrightarrow{\Prob_b}b$.

\paragraph{Step 2: Covariance matrix.} Since $\E\|\psi_i\psi_i^\top\|_{\mathrm{op}}\le L^2<\infty$, the weak law of large numbers gives $\widehat Q_n\xrightarrow{\Prob_b}Q$. 

\paragraph{Step 3: Moment convergence.} Decompose $\widehat m_n=S_n+R_n$ with $S_n :=\frac{1}{n}\sum_i\psi_i z_iw_b(t_i)$ and $R_n :=\frac{1}{n}\sum_i\psi_i z_i(w_{\btil}(t_i)-w_b(t_i))$.

\textbf{Analysis of $S_n$.} Conditional on $(\psi_i,\theta_i)$, the drift is the fixed scalar $v_i=\psi_i^\top\theta_i$, so Theorem~\ref{thm:unbiased} gives $\E_b[z_iw_b(t_i)\mid\psi_i,\theta_i]=\psi_i^\top\theta_i$. Multiplying by $\psi_i$ and applying the exogeneity condition $\E_b[\theta_i\mid\psi_i]=\theta^\star$,
\begin{align*}
    \E_b[\psi_iz_iw_b(t_i)] \;=\; \E_b[\psi_i\psi_i^\top\theta_i] \;=\; \E_b[\psi_i\psi_i^\top\E_b[\theta_i\mid\psi_i]] \;=\; Q\theta^\star.
\end{align*}
Also $\|\psi_iz_iw_b(t_i)\|\le L|w_b(t_i)|\le LH_K(t_i)$ for $K=[b/2,3b/2]$, which is integrable by \Cref{lem:envelope}. Hence $S_n\xrightarrow{\Prob_b}Q\theta^\star$ by the weak law of large numbers.

\textbf{Analysis of $R_n$.} For $\delta\in(0,b/2]$, define $r_\delta$ as in the scalar proof. On $\{|\btil-b|\le\delta\}$,
$\|R_n\|\le L\frac{1}{n}\sum_i r_\delta(t_i)$. Choose $\delta$ with $L\,\E_b[r_\delta(t_i)]<\varepsilon/2$. Then
\begin{align*}
    \Prob_b(\|R_n\|>\varepsilon)
    \;\le\; \Prob_b(|\btil-b|>\delta)+\Prob_b(\frac{L}{n}\sum_i r_\delta(t_i)>\varepsilon),
\end{align*}
and both terms vanish by Theorem~\ref{thm:two-scale} and the weak law of large numbers respectively. So $R_n\xrightarrow{\Prob_b}0$.

\paragraph{Step 4: Continuous mapping.} Combining $\widehat Q_n\xrightarrow{\Prob_b}Q$ with $\widehat m_n=S_n+R_n\xrightarrow{\Prob_b}Q\theta^\star$ and applying the continuous mapping theorem to $(\widehat Q_n,\widehat m_n)\mapsto\widehat Q_n^{-1}\widehat m_n$ on the open set $\{Q:\det Q>0\}$,
\begin{align*}
    \thetahat \;=\; \widehat Q_n^{-1}\widehat m_n \;\xrightarrow{\Prob_b}\; Q^{-1}(Q\theta^\star) \;=\; \theta^\star.
\end{align*}
For every fixed $x$, $\thetahat^\top\phi(x)\xrightarrow{\Prob_b}{\theta^\star}^\top\phi(x)$ by the continuous-mapping theorem. \qed

\section{Additional Results and Discussion} \label{app:additional_results}

In this section, we provide additional results and further discussion, mainly related to some of the assumptions made in the paper.
\subsection{Non-decision time} \label{app:no_decision_time}
In some applications, it is natural to assume that the decision-maker takes some time to parse the two alternatives before starting the evidence accumulation that is modeled by means of the DDM. In particular, the classical \textbf{extended} \ddm{} \citep{ratcliff2002estimating} models such non-decision time as an additive term. More specifically, the observed response time, now denoted by $S$, is given by 
\begin{align}
S=T+A,
\label{eq:observed-time-with-ndt}
\end{align}
where $T$ is the latent DDM first-passage time and $A\geq 0$ denotes the non-decision time. It is worth noting that, here, $A$ is not assumed to be a function of the alternatives; rather, it may be either deterministic or stochastic. 

In this subsection, we show how our model can be adjusted to handle a deterministic additive non-decision time. When $A$ is allowed to be random or heterogeneous, additional complexity arises since, in that case, one must learn the boundary, the heterogeneous drift, and the random non-decision response time jointly, while potentially having only one data point per labeler. One can use our results to derive bounds on the estimation error in that case (see Section~\ref{app:boundary_heterogeneity}); we leave this direction for future work.

Before stating our results, it is useful to highlight that, in settings such as LLM responses, longer responses are typically harder to parse and compare. In this case, one can take a different approach than modeling the non-decision response time as an additive term and instead incorporate it into the drift variable. Our linear model in \Cref{sec:linear} already handles this scenario: it suffices to add one additional coordinate to $\psi_i$, whose corresponding coordinate in $\theta_i$ models this effect.

\subsubsection{Extended \ddm{} with non-decision time}
Throughout this subsection, we assume $A=a>0$ is deterministic. Let
\begin{align*}
L_T(\lambda):=\mathbb{E}\!\left[e^{-\lambda T}\right],
\qquad
L_S(\lambda):=\mathbb{E}\!\left[e^{-\lambda S}\right] = e^{-a\lambda}L_T(\lambda).
\end{align*}
First, we argue why we need to adjust the estimator of the boundary, and why the two-scale estimator \eqref{eq:btil-def} fails. Note that, the population version of this estimator satisfies
\begin{align*}
\frac{\log L_S(\lambda)-\log L_S(4\lambda)}
{\sqrt{2\lambda}}
=
\frac{3a}{\sqrt{2}}\sqrt{\lambda}
+
\frac{\log L_T(\lambda)-\log L_T(4\lambda)}
{\sqrt{2\lambda}},
\end{align*}
and even though we know the second term converges to $b$ as $\lambda \to \infty$, the whole estimator would diverge to $+\infty$ at rate $\sqrt{\lambda}$. This observation motivates a three-scale estimator that cancels the linear non-decision-time term.
\begin{proposition}[Three-scale boundary and non-decision-time estimators]
\label{prop:three-scale-ndt-estimator}
Assume that $|V|\le M$ almost surely and that $\{\lambda_n\}_n$ satisfies \eqref{eq:lambda-rate}. Moreover, assume that $S=T+a$, where $a\geq 0$ is deterministic.
Let
\begin{align*}
\widehat L_n^S(\lambda)
:=
\frac{1}{n}\sum_{i=1}^n e^{-\lambda S_i},
\qquad
\widehat F_n(\lambda):=\log \widehat L_n^S(\lambda).
\end{align*}
Define
\begin{align*}
\widehat b_n^{\mathrm{NDT}}
:=
\frac{
\frac{5}{2}\widehat F_n(\lambda_n)
-
4\widehat F_n(4\lambda_n)
+
\frac{3}{2}\widehat F_n(9\lambda_n)
}
{\sqrt{2\lambda_n}},
\end{align*}
and
\begin{align*}
\widehat a_n
:=
\frac{
-\frac{1}{2}\widehat F_n(\lambda_n)
+
\widehat F_n(4\lambda_n)
-
\frac{1}{2}\widehat F_n(9\lambda_n)
}
{\lambda_n}.
\end{align*}
Then,
\begin{align*}
\widehat b_n^{\mathrm{NDT}}
\xrightarrow{{P}}
b,
\qquad
\widehat a_n
\xrightarrow{{P}}
a.
\end{align*}
\end{proposition}
\begin{proof}
The proof follows the logic of the proof of \Cref{thm:two-scale} in Appendix~\ref{app:proof-two-scale}. Using the expansion of $\Lb(\lambda)$ developed there, it is straightforward to verify that the population versions of $\widehat b_n^{\mathrm{NDT}}$ and $\widehat a_n$ converge to the corresponding parameters as $\lambda_n$ grows. To obtain convergence in population, since $A=a$ is deterministic, we can again use the result developed in \Cref{thm:one-scale}, as in the proof of \Cref{thm:two-scale}.
\end{proof}
We conclude this subsection by deriving an analogue of \Cref{thm:plugin-scalar} and designing a consistent estimator of $\mu$ while accounting for the non-decision time. The main idea is to remove the effect of the non-decision time by using the consistent estimator $\widehat a_n$ developed above. However, since $\widehat a_n$ is itself an estimator, it may even exceed some observed response times. In particular, the next result shows how to address this difficulty carefully.
\begin{theorem} \label{thm:consistent_estimator_NDT}
Under the premise of \Cref{prop:three-scale-ndt-estimator}, let $\varepsilon_n\downarrow 0$ be any deterministic sequence satisfying
\begin{align}
\text{Pr}\!\left(|\widehat a_n-a|>\varepsilon_n/2\right)\to 0.
\label{eq:ndt-trimming-sequence-condition}
\end{align}
Define
\begin{align}
\widehat \mu_{n,\varepsilon}^{\mathrm{NDT}}
:=
\frac{1}{n}
\sum_{i=1}^n
Z_i
w_{\widehat b_n^{\mathrm{NDT}}}(S_i-\widehat a_n)
\mathbf{1}\{S_i-\widehat a_n\geq \varepsilon_n\}.
\label{eq:ndt-corrected-drift-estimator}
\end{align}
Then,
$
\widehat \mu_{n,\varepsilon}^{\mathrm{NDT}}
\xrightarrow{{P}}
\mu =\mathbb{E}[V].
$   
\end{theorem}
\begin{proof}
In this proof, we make use of the following lemma.
\begin{lemma} \label{lemma:time_shift_envelop}
Let $T$ be the first passage time of \ddm{}, i.e., the response time under zero non-decision time. Then, we have
\begin{align}
& \lim_{\varepsilon\downarrow 0}
\mathbb{E}
\!\left[
|w_b(T)|\mathbf{1}\{T\leq 2\varepsilon\}
\right]
= 0
\label{eq:small-time-truncation-envelope}\\
& \lim_{\rho\downarrow 0}
\limsup_{\varepsilon\downarrow 0}
\mathbb{E}
\!\left[
\sup_{\substack{|\alpha-b|\leq \rho\\ |u|\leq \varepsilon\\ T+u \geq \varepsilon}}
\left|
w_{\alpha}(T+u)-w_b(T)
\right|
\right]
= 0.
\label{eq:time-shift-continuity-envelope}
\end{align}
\end{lemma}
We first show how to prove \Cref{thm:consistent_estimator_NDT} using this lemma, and then prove the lemma itself. 
Define
$
u_n:=a-\widehat a_n
$
as the error of the non-decision time estimator $a$ at time $n$. Then, we have
\begin{align*}
S_i-\widehat a_n
=
T_i+a-\widehat a_n
=
T_i+u_n.
\end{align*}
Let $\rho>0$ be fixed and define the good event
\begin{align}
E_{n,\rho}
:=
\left\{
|\widehat b_n^{\mathrm{NDT}}-b|\leq \rho,
\quad
|\widehat a_n-a|\leq \varepsilon_n/2
\right\}.
\label{eq:ndt-corrected-good-event}
\end{align}
By $\widehat b_n^{\mathrm{NDT}}
\xrightarrow{{P}} b$ and \eqref{eq:ndt-trimming-sequence-condition},
\begin{align*}
\text{Pr}(E_{n,\rho})\to 1
\qquad
\text{for every fixed } \rho>0.
\end{align*}
Decompose
\begin{align*}
\widehat \mu_{n,\varepsilon}^{\mathrm{NDT}}-\mu
=
A_n+B_n,
\end{align*}
where
\begin{align*}
A_n
:=
\frac{1}{n}\sum_{i=1}^n Z_iw_b(T_i)-\mu
\end{align*}
and
\begin{align*}
B_n
:=
\frac{1}{n}
\sum_{i=1}^n Z_i
\left[
w_{\widehat b_n^{\mathrm{NDT}}}(T_i+u_n)
\mathbf{1}\{T_i+u_n\geq \varepsilon_n\}
-
w_b(T_i)
\right].
\end{align*}
We already know $A_n\xrightarrow{{P}}0$. Hence, it suffices to show $B_n\xrightarrow{{P}}0$. 

Notice that, on the event $E_{n,\rho}$, if $T_i+u_n<\varepsilon_n$, then $T_i<3\varepsilon_n/2$. Hence, using $|Z_i|=1$, on the event $E_{n,\rho}$, we have
\begin{align*}
|B_n|
&\leq
\frac{1}{n}
\sum_{i=1}^n
\sup_{\substack{|\alpha-b|\leq \rho\\ |u|\leq \varepsilon_n/2\\ T_i+u \geq \varepsilon_n 0}}
\left|
w_{\alpha}(T_i+u)-w_b(T_i)
\right|
+
\frac{1}{n}
\sum_{i=1}^n
|w_b(T_i)|
\mathbf{1}\{T_i<3\varepsilon_n/2\}.
\end{align*}
As a result, for any $\rho$, we have
\begin{align}
\text{Pr}(|B_n|>\delta)&  \leq 
\text{Pr}(E_{n,\rho}) + 
\text{Pr} \left( 
\frac{1}{n}
\sum_{i=1}^n
\sup_{\substack{|\alpha-b|\leq \rho\\ |u|\leq \varepsilon_n/2\\ T_i+u \geq \varepsilon_n 0}}
\left|
w_{\alpha}(T_i+u)-w_b(T_i)
\right| \geq \frac{\delta}{2}
\right) \nonumber \\
& + 
\text{Pr} \left( 
\frac{1}{n}
\sum_{i=1}^n
|w_b(T_i)|
\mathbf{1}\{T_i<3\varepsilon_n/2\} \geq \frac{\delta}{2}
\right) \nonumber  \\
& \overset{\text{(Markov's inequality)}}{\leq} \text{Pr}(E_{n,\rho}) + \frac{2}{\delta} \mathbb{E} \left[ 
\sup_{\substack{|\alpha-b|\leq \rho\\ |u|\leq \varepsilon_n/2\\ T_i+u \geq \varepsilon_n 0}}
\left|
w_{\alpha}(T+u)-w_b(T)
\right|
\right ] \nonumber \\
& + \frac{2}{\delta}  \mathbb{E} \left[ 
|w_b(T)|
\mathbf{1}\{T<3\varepsilon_n/2\} 
\right]. \label{eqn:bound_prob_B_n}
\end{align}
Now, we claim that this implies $\Pr(|B_n|>\delta)$ goes to zero as $n\to\infty$. Suppose not. Then there exist $\zeta>0$ and an infinite subsequence of values of $n$ such that  $\text{Pr}(|B_n|>\delta) > \zeta$.  

Notice that, by \Cref{lemma:time_shift_envelop}, the last term in \eqref{eqn:bound_prob_B_n} goes to zero as $n\to\infty$. Also, again by \Cref{lemma:time_shift_envelop}, one can choose $\rho$ small enough to make the second term arbitrarily small. Finally, for any fixed $\rho$, the first term goes to zero as $n\to\infty$. Taken together, these facts yield a contradiction and hence complete the proof of \Cref{thm:consistent_estimator_NDT}. It remains to prove \Cref{lemma:time_shift_envelop}.

\paragraph{Proof of \Cref{lemma:time_shift_envelop}}
Let us define
\[
\Delta_{\rho,\varepsilon}(t)
:=
\sup_{\substack{|\alpha-b|\le \rho\\ |u|\le \varepsilon\\ t+u\ge \varepsilon}}
\left|w_\alpha(t+u)-w_b(t)\right|.
\]
Let $H_K$ be the envelope from \Cref{lem:envelope}, so that
\[
\sup_{\beta\in K}|w_\beta(t)|\le H_K(t),
\qquad
\mathbb E[H_K(T)]<\infty.
\]
The limit \eqref{eq:small-time-truncation-envelope} follows immediately from dominated convergence:
\[
|w_b(T)|\mathbf 1\{T\le 2\varepsilon\}
\le
H_K(T),
\]
and
\[
H_K(T)\mathbf 1\{T\le 2\varepsilon\}
\to 0
\qquad
\text{almost surely as }\varepsilon\downarrow0.
\]
Since $H_K(T)$ is integrable, dominated convergence gives
\[
\lim_{\varepsilon\downarrow0}
\mathbb E[
|w_b(T)|\mathbf 1\{T\le 2\varepsilon\}
]
=
0.
\]

We now prove \eqref{eq:time-shift-continuity-envelope}. Decompose
\[
\mathbb E[\Delta_{\rho,\varepsilon}(T)]
=
\mathbb E[
\Delta_{\rho,\varepsilon}(T)\mathbf 1\{T>2\varepsilon\}
]
+
\mathbb E[
\Delta_{\rho,\varepsilon}(T)\mathbf 1\{T\le 2\varepsilon\}
].
\]

First consider the event $\{T>2\varepsilon\}$. On this event, for every $|u|\le \varepsilon$,
\[
T+u\ge T-\varepsilon>T/2.
\]
Assume also $\varepsilon\le 1$. Then
\[
T+u\le T+1.
\]
Define
\[
\widetilde H_K(t)
:=
H_K(t)+\sup_{s\in[t/2,t+1]}H_K(s).
\]
By \Cref{lem:envelope}, $\widetilde H_K(T)$ is integrable. Indeed, for small $t$,
\[
\sup_{s\in[t/2,t+1]}H_K(s)
\le
C(1+t^{-1}),
\]
and the small-time integrability follows from \Cref{lem:envelope}. For $t$ bounded away from zero,
$H_K$ is bounded by the middle-region and large-time pieces of \Cref{lem:envelope}.
Therefore, on $\{T>2\varepsilon\}$,
\[
\Delta_{\rho,\varepsilon}(T)
\le
\widetilde H_K(T).
\]
For each fixed $t>0$,
\[
\lim_{\varepsilon\downarrow0}
\Delta_{\rho,\varepsilon}(t)\mathbf 1\{t>2\varepsilon\}
=
r_\rho(t),
\]
where
\[
r_\rho(t):=
\sup_{|\alpha-b|\le \rho}|w_\alpha(t)-w_b(t)|.
\]
The pointwise convergence follows from joint continuity of $(\alpha,t)\mapsto w_\alpha(t)$ on compact subsets of $K\times(0,\infty)$.
Thus dominated convergence gives
\[
\lim_{\varepsilon\downarrow0}
\mathbb E[
\Delta_{\rho,\varepsilon}(T)\mathbf 1\{T>2\varepsilon\}
]
=
\mathbb E[r_\rho(T)].
\]
By \Cref{lem:envelope},
\[
r_\rho(T)\le 2H_K(T),
\]
and
\[
r_\rho(T)\to0
\qquad
\text{almost surely as }\rho\downarrow0.
\]
Another application of dominated convergence gives
\[
\lim_{\rho\downarrow0}\mathbb E[r_\rho(T)]=0.
\]

It remains to control the event $\{T\le 2\varepsilon\}$. On this event, if
$|u|\le\varepsilon$ and $T+u\ge\varepsilon$, then
\[
\varepsilon\le T+u\le 3\varepsilon.
\]
For $\varepsilon$ small enough that $3\varepsilon\le\tau_0$, \Cref{lem:envelope} gives
\[
|w_\alpha(T+u)|
\le
C_K(1+\varepsilon^{-1})
\qquad
\text{for all }\alpha\in K.
\]
Therefore
\[
\Delta_{\rho,\varepsilon}(T)\mathbf 1\{T\le2\varepsilon\}
\le
C_K(1+\varepsilon^{-1})\mathbf 1\{T\le2\varepsilon\}
+
H_K(T)\mathbf 1\{T\le2\varepsilon\}.
\]
Taking expectations,
\[
\mathbb E[
\Delta_{\rho,\varepsilon}(T)\mathbf 1\{T\le2\varepsilon\}
]
\le
C_K(1+\varepsilon^{-1})\mathbb P_b(T\le2\varepsilon)
+
\mathbb E[
H_K(T)\mathbf 1\{T\le2\varepsilon\}
].
\]
The second term converges to zero by dominated convergence.

For the first term, let $g_b$ be the marginal density of $T$ under boundary $b$. The proof of \Cref{lem:envelope} gives, for small $t$,
\[
g_b(t)
\le
C t^{-3/2}\exp\!\left(-\frac{b^2}{2t}\right).
\]
Hence
\[
\mathbb P_b(T\le2\varepsilon)
\le
C\int_0^{2\varepsilon}
t^{-3/2}\exp\!\left(-\frac{b^2}{2t}\right)\,dt.
\]
Using the change of variables
\[
x=\frac{b^2}{2t},
\qquad
t=\frac{b^2}{2x},
\qquad
dt=-\frac{b^2}{2x^2}\,dx,
\]
we get
\[
\int_0^{2\varepsilon}
t^{-3/2}\exp\!\left(-\frac{b^2}{2t}\right)\,dt
=
C_b
\int_{b^2/(4\varepsilon)}^\infty
x^{-1/2}e^{-x}\,dx.
\]
For large $a$,
\[
\int_a^\infty x^{-1/2}e^{-x}\,dx
=
O(a^{-1/2}e^{-a}).
\]
Therefore
\[
\mathbb P_b(T\le2\varepsilon)
=
O\!\left(
\varepsilon^{1/2}
\exp\!\left(-\frac{b^2}{4\varepsilon}\right)
\right).
\]
Thus
\[
(1+\varepsilon^{-1})\mathbb P_b(T\le2\varepsilon)
=
O\!\left(
\varepsilon^{-1/2}
\exp\!\left(-\frac{b^2}{4\varepsilon}\right)
\right)
\to0.
\]
Hence
\[
\lim_{\varepsilon\downarrow0}
\mathbb E[
\Delta_{\rho,\varepsilon}(T)\mathbf 1\{T\le2\varepsilon\}
]
=
0
\]
uniformly over $\rho\in(0,b/2]$.

Combining the two regions gives
\[
\limsup_{\varepsilon\downarrow0}
\mathbb E[\Delta_{\rho,\varepsilon}(T)]
\le
\mathbb E[r_\rho(T)].
\]
Letting $\rho\downarrow0$ yields
\[
\lim_{\rho\downarrow0}\limsup_{\varepsilon\downarrow0}
\mathbb E[\Delta_{\rho,\varepsilon}(T)]
=
0.
\]
This completes the proof of \Cref{lemma:time_shift_envelop} and hence the proof of \Cref{thm:consistent_estimator_NDT}.
\end{proof}

\subsection{Robustness to boundary heterogeneity} \label{app:boundary_heterogeneity}
While we allow for heterogeneous drifts, the common-boundary assumption allows us to learn a single boundary from the aggregated data and then use it to estimate the average drift. However, if we allow for both heterogeneous drifts and arbitrary heterogeneous boundaries, learning the boundary from one observation per labeler is not possible. That said, one may still ask what our estimator converges to when the common-boundary assumption fails. This is important because it can be viewed as a robustness question. The next result shows that, in this case, the estimator actually estimates the lower endpoint of the boundary distribution. Throughout this section, we do not consider the non-decision response time. 
\begin{proposition}
\label{prop:unobserved-boundary-lower-endpoint}
Let $(B,V)$ be an arbitrary random pair of boundary and drift, satisfying
\begin{align*}
0<b_- \leq B\leq b_+<\infty,
\qquad
|V|\leq M<\infty
\qquad
\text{almost surely},
\end{align*}
where $b_-:=\operatorname{ess\,inf} B$.
Let $T$ be the \ddm{} response time conditionally on $(B,V)$. Define the population Laplace transform $L(\lambda):=\mathbb{E}\!\left[e^{-\lambda T}\right]$.
Then
\begin{align}
-\frac{\log L(\lambda)}{\sqrt{2\lambda}}
\longrightarrow b_-
\qquad
\text{as } \lambda\to\infty.
\label{eq:heterogeneous-boundary-lower-endpoint-limit}
\end{align}
\end{proposition}
\begin{proof}
Recall that
\begin{align*}
L(\lambda)
=
\mathbb{E}\!\left[
\frac{\cosh(BV)}
{\cosh\!\left(B\sqrt{V^2+2\lambda}\right)}
\right].
\end{align*}
We first prove the lower bound on the logarithmic rate. Similar to the derivation of \eqref{eq:app-Lb-upper}, we have
\begin{align*}
L(\lambda)
\leq
2e^{b_+M}
\exp\!\left(-b_-\sqrt{2\lambda}\right).
\end{align*}
Therefore
\begin{align*}
-\log L(\lambda)
\geq
b_-\sqrt{2\lambda}
-
\log\!\left(2e^{b_+M}\right).
\end{align*}
Dividing by $\sqrt{2\lambda}$ and taking the limit inferior yields
\begin{align}
\liminf_{\lambda\to\infty}
\left(
-\frac{\log L(\lambda)}{\sqrt{2\lambda}}
\right)
\geq b_-.
\label{eq:heterogeneous-boundary-liminf}
\end{align}

We next prove the reverse inequality. Fix $\varepsilon>0$. By the definition of the essential infimum,
\begin{align*}
p_{\varepsilon}:=\mathbb{P}\!\left(B\leq b_-+\varepsilon\right)>0.
\end{align*}
Now, using this, and similar to the derivation of \eqref{eq:app-Lb-lower}, we have
Hence
\begin{align*}
L(\lambda)
&\geq
p_{\varepsilon}
\exp\!\left(-(b_-+\varepsilon)\sqrt{M^2+2\lambda}\right).
\end{align*}
Therefore
\begin{align*}
-\log L(\lambda)
\leq
(b_-+\varepsilon)\sqrt{M^2+2\lambda}
-
\log p_{\varepsilon}.
\end{align*}
Dividing by $\sqrt{2\lambda}$ and taking the limit superior gives
\begin{align}
\limsup_{\lambda\to\infty}
\left(
-\frac{\log L(\lambda)}{\sqrt{2\lambda}}
\right)
\leq
b_-+\varepsilon.
\label{eq:heterogeneous-boundary-limsup-epsilon}
\end{align}
Since $\varepsilon>0$ was arbitrary,
\begin{align}
\limsup_{\lambda\to\infty}
\left(
-\frac{\log L(\lambda)}{\sqrt{2\lambda}}
\right)
\leq b_-.
\label{eq:heterogeneous-boundary-limsup}
\end{align}
Combining \eqref{eq:heterogeneous-boundary-liminf} and \eqref{eq:heterogeneous-boundary-limsup} proves \eqref{eq:heterogeneous-boundary-lower-endpoint-limit}.
\end{proof}
Proposition \ref{prop:unobserved-boundary-lower-endpoint} implies that arbitrary unobserved boundary heterogeneity is not harmless. That said, this also means that when the boundary is not homogeneous but still lies in a small interval, our estimator returns an approximately accurate estimate of it. One can derive bounds on the final estimation error as a function of $b_{+}-b_{-}$.

One final remark is that when boundary heterogeneity is associated with observed strata, such as task type, annotation interface, demographic pool, or any other observed grouping variable, we can group users accordingly, learn one boundary for each group, and again apply the results derived under the homogeneous-boundary case.
\subsection{Relaxing the bounded drifts assumption}
\label{subsec:light-tailed-drifts}

In this section, we relax the bounded-drift assumption to require only bounded exponential moments. This relaxation ensures our estimators provably accommodate Gaussian drift priors. Our contributions in this section are summarized as follows:
\begin{itemize}[leftmargin=1.5em,itemsep=1pt,topsep=1pt]
    \item In \Cref{thm:one-scale} and \Cref{thm:two-scale}, we relax the almost everywhere (a.e.) boundedness assumption, $|V| \leq M$ a.e., to the moment condition $\mathbb{E}\!\left[(1+V^4)\cosh(bV)\right]<\infty$.
    \item In \Cref{thm:linear-main}, we similarly relax the assumption $\|\psi_i\|\le L$ and $\|\theta_i\|\le R$ a.e., to $\mathbb{E}\!\left[\left(1 + \|\psi_i\|^2 + \|\psi_i\|^4\|\theta_i\|^4\right)\cosh\!\left(b\psi_i^{\top}\theta_i\right)\right]<\infty$.
\end{itemize}

Our first result recovers consistent boundary estimate under the relaxed assumption.

\begin{proposition}[Empirical boundary consistency under exponential moments]
\label{cor:light-tailed-empirical-boundary-consistency}
Let $(V_i,T_i)_{i=1}^n$ be i.i.d. as follows: $V_i\sim F^*$,
and conditional on $V_i=v$, $T_i$ is the first-passage time of a DDM
with common boundary $b>0$ and drift $v$. Conditional on $(V_i)_{i=1}^n$,
the DDM paths are independent. Assume
\begin{align}\label{eq:cosh-moment-condition}
C_0 =\mathbb{E}[\cosh(bV)]<\infty.
\end{align}
Then, 
\begin{align}
-\frac{\log L_b(\lambda)}{\sqrt{2\lambda}}
\to b.
\label{eq:light-tailed-one-scale-boundary-limit}
\end{align}
Furthermore, suppose $\{\lambda_n\}_n$ satisfies \eqref{eq:lambda-rate}. Then,
\begin{align}
\frac{\widehat L_n(\lambda_n)}
{L_b(\lambda_n)}
\xrightarrow{{P}}
1, \quad
-\frac{\log \widehat L_n(\lambda_n)}
{\sqrt{2\lambda_n}}
\xrightarrow{{P}}
b.
\label{eq:light-tailed-relative-empirical-laplace}
\end{align}
The same conclusion holds for the Richardson estimator $\widehat B_n$ in \eqref{eq:btil-def}.
\end{proposition}
\begin{proof}
We first show 
\begin{align}
e^{b\sqrt{2\lambda}}L_b(\lambda)
\to
2C_0
\qquad
\text{as } \lambda\to\infty.
\label{eq:light-tailed-laplace-tail-equivalent}
\end{align}

For each fixed $v\in\mathbb{R}$,
\begin{align*}
\sqrt{v^2+2\lambda}-\sqrt{2\lambda}\to 0
\qquad
\text{as } \lambda\to\infty.
\end{align*}
Using $\cosh(x)=(e^x+e^{-x})/2$, we obtain the pointwise limit
\begin{align*}
\frac{
e^{b\sqrt{2\lambda}}\cosh(bv)
}
{
\cosh\!\left(b\sqrt{v^2+2\lambda}\right)
}
\to
2\cosh(bv).
\end{align*}
For domination, since $\cosh(y)\geq e^y/2$ for $y\geq 0$,
\begin{align*}
\frac{
e^{b\sqrt{2\lambda}}\cosh(bv)
}
{
\cosh\!\left(b\sqrt{v^2+2\lambda}\right)
}
\leq
2\cosh(bv)
\exp\!\left(
-b\{\sqrt{v^2+2\lambda}-\sqrt{2\lambda}\}
\right) \leq
2\cosh(bv).
\end{align*}
The dominating function $2\cosh(bV)$ is integrable by \eqref{eq:cosh-moment-condition}. Dominated convergence therefore gives \eqref{eq:light-tailed-laplace-tail-equivalent}. Taking logarithms in \eqref{eq:light-tailed-laplace-tail-equivalent} yields
\begin{align*}
\log L_b(\lambda)
=
-b\sqrt{2\lambda}
+
\log(2C_0)
+
o(1).
\end{align*}
Dividing by $\sqrt{2\lambda}$ proves \eqref{eq:light-tailed-one-scale-boundary-limit}.

Since $0\leq e^{-\lambda T}\leq 1$,
\begin{align*}
\operatorname{Var}\!\left(e^{-\lambda T}\right)
\leq
\mathbb{E}\!\left[e^{-2\lambda T}\right]
=
L_b(2\lambda).
\end{align*}
Therefore
\begin{align*}
\mathbb{E}
\!\left[
\left(
\frac{\widehat L_n(\lambda_n)}{L_b(\lambda_n)}
-
1
\right)^2
\right]
=
\frac{\operatorname{Var}(e^{-\lambda_nT})}
{nL_b(\lambda_n)^2}
\leq
\frac{L_b(2\lambda_n)}
{nL_b(\lambda_n)^2}.
\end{align*}
By \eqref{eq:light-tailed-laplace-tail-equivalent},
\begin{align*}
L_b(2\lambda)
=
O\!\left(e^{-b\sqrt{4\lambda}}\right),
\qquad
L_b(\lambda)^2
=
\Theta\!\left(e^{-2b\sqrt{2\lambda}}\right).
\end{align*}
Hence
\begin{align}
\frac{L_b(2\lambda_n)}
{nL_b(\lambda_n)^2}
=
O\!\left(
\frac{
\exp\!\left(2b(\sqrt{2}-1)\sqrt{\lambda_n}\right)
}
{n}
\right).
\label{eq:light-tailed-relative-variance-order}
\end{align}
Condition \eqref{eq:lambda-rate} implies that the right-hand side of \eqref{eq:light-tailed-relative-variance-order} converges to zero. This proves \eqref{eq:light-tailed-relative-empirical-laplace}. The Richardson estimator uses the same argument at scales $\lambda_n$ and $4\lambda_n$, and $2\sqrt{\lambda_n}=o(\log n)$ still holds. Thus $\widehat B_n$ also converges in probability to $b$.
\end{proof}
It is worth noting that the scalar exponential-moment condition includes Gaussian drifts. Indeed, if $V\sim N(m,\sigma^2)$, then
\begin{align*}
\mathbb{E}[\cosh(bV)]
=
\frac{1}{2}
\left(
e^{bm+b^2\sigma^2/2}
+
e^{-bm+b^2\sigma^2/2}
\right)
<\infty.
\end{align*}
This result implies that the estimator $\widehat B_n$, which we developed earlier, is consistent under the lighter moment-type assumption \eqref{eq:cosh-moment-condition}, rather than under the boundedness assumption. That said, to establish the second part of \Cref{thm:two-scale}, which shows that this estimator enjoys a better rate than the one-scale estimator, we need a slightly stronger assumption, which still holds for Gaussian drifts. The next proposition formalizes this.

\begin{proposition}
\label{prop:light-tailed-richardson-expansion}
Under the premise of \Cref{cor:light-tailed-empirical-boundary-consistency}, assume
\begin{align}
\mathbb{E}
\!\left[
(1+V^4)\cosh(bV)
\right]
<\infty.
\label{eq:fourth-cosh-moment-condition}
\end{align}
Then, \eqref{eqn:estimators_convergence_rate} holds, meaning that the population version of $\widehat B_n$ enjoys a faster convergence rate than the one-scale estimator.
\end{proposition}
\begin{proof}
Recall from the proof of \Cref{thm:two-scale} that
\begin{align*}
C_0:=\mathbb{E}[\cosh(bV)],
\qquad
C_2:=\mathbb{E}[V^2\cosh(bV)].
\end{align*}
Let
\begin{align*}
s:=\sqrt{2\lambda},
\qquad
\delta_s(v):=\sqrt{s^2+v^2}-s.
\end{align*}
Then
\begin{align*}
\delta_s(v)
=
\frac{v^2}{\sqrt{s^2+v^2}+s},
\qquad
0\leq \delta_s(v)\leq \frac{v^2}{2s}.
\end{align*}
Also,
\begin{align*}
\frac{1}{\cosh(b\sqrt{s^2+v^2})}
=
2e^{-bs}e^{-b\delta_s(v)}
\left(1+\rho_s(v)\right),
\end{align*}
where
\begin{align*}
|\rho_s(v)|
\leq
e^{-2b\sqrt{s^2+v^2}}
\leq
e^{-2bs}.
\end{align*}
Therefore
\begin{align}
L_b(\lambda)
=
2e^{-bs}
\mathbb{E}
\!\left[
\cosh(bV)e^{-b\delta_s(V)}
\right]
+
O(e^{-3bs}C_0).
\label{eq:laplace-after-inverse-cosh-expansion}
\end{align}

We now expand the remaining expectation. Since $0\leq \delta_s(V)\leq V^2/(2s)$ and
\begin{align*}
\left|
e^{-b\delta_s(V)}-1+b\delta_s(V)
\right|
\leq
\frac{b^2\delta_s(V)^2}{2}
\leq
\frac{b^2V^4}{8s^2},
\end{align*}
we have
\begin{align}
\mathbb{E}
\!\left[
\cosh(bV)e^{-b\delta_s(V)}
\right]
=
C_0
-
b\mathbb{E}
\!\left[
\cosh(bV)\delta_s(V)
\right]
+
O(s^{-2}).
\label{eq:expectation-after-taylor-expansion}
\end{align}
Next,
\begin{align*}
0
\leq
\frac{V^2}{2s}-\delta_s(V)
=
\frac{V^2(\sqrt{s^2+V^2}-s)}
{2s(\sqrt{s^2+V^2}+s)}
\leq
\frac{V^4}{8s^3}.
\end{align*}
Thus
\begin{align}
\mathbb{E}
\!\left[
\cosh(bV)\delta_s(V)
\right]
=
\frac{C_2}{2s}
+
O(s^{-3}).
\label{eq:delta-s-expectation-expansion}
\end{align}
Combining \eqref{eq:laplace-after-inverse-cosh-expansion}, \eqref{eq:expectation-after-taylor-expansion}, and \eqref{eq:delta-s-expectation-expansion} gives
\begin{align*}
L_b(\lambda)
=
2e^{-bs}
\left(
C_0
-
\frac{bC_2}{2s}
+
O(s^{-2})
\right).
\end{align*}
Taking logarithms,
\begin{align*}
\log L_b(\lambda)
&=
-bs+\log(2C_0)
+
\log
\left(
1-\frac{bC_2}{2C_0s}+O(s^{-2})
\right)
\notag\\
&=
-bs+\log(2C_0)
-
\frac{bC_2}{2C_0s}
+
O(s^{-2}).
\end{align*}
Since $s=\sqrt{2\lambda}$, implies
\begin{align}
\log L_b(\lambda)
=
-b\sqrt{2\lambda}
+
\log(2C_0)
-
\frac{bC_2}{2C_0\sqrt{2\lambda}}
+
O(\lambda^{-1}).
\label{eq:light-tailed-richardson-log-expansion}
\end{align}

Let
\begin{align*}
c_0:=\log(2C_0),
\qquad
c_1:=-\frac{bC_2}{2C_0}.
\end{align*}
Then \eqref{eq:light-tailed-richardson-log-expansion} can be written as
\begin{align}
\log L_b(\lambda)
=
-bs+c_0+\frac{c_1}{s}+O(s^{-2}).
\label{eq:richardson-expansion-in-s}
\end{align}
At scale $4\lambda$, the corresponding value of $s$ is $2s$, and hence
\begin{align}
\log L_b(4\lambda)
=
-2bs+c_0+\frac{c_1}{2s}+O(s^{-2}).
\label{eq:richardson-expansion-at-four-lambda}
\end{align}
Subtracting \eqref{eq:richardson-expansion-at-four-lambda} from \eqref{eq:richardson-expansion-in-s} gives
\begin{align*}
\log L_b(\lambda)-\log L_b(4\lambda)
=
bs+\frac{c_1}{2s}+O(s^{-2}).
\end{align*}
Dividing by $s=\sqrt{2\lambda}$ proves the rate for the Richardson estimator. Finally,
\begin{align}
-\frac{\log L_b(\lambda)}{s}-b
=
-\frac{c_0}{s}
-
\frac{c_1}{s^2}
+
O(s^{-3}).
\label{eq:one-scale-bias-expansion}
\end{align}
Because $C_0=\mathbb{E}[\cosh(bV)]\geq 1$, we have $c_0=\log(2C_0)>0$. Hence the leading term in \eqref{eq:one-scale-bias-expansion} is nonzero, and the other rate follows.
\end{proof}
Finally, we can similarly relax the boundedness assumption in the linear model.

\begin{proposition}[Linear preference consistency under integrability conditions]
\label{prop:linear-light-tailed-consistency}
Consider the setting of linear preferences in \Cref{sec:linear}.  
Assume
\begin{align}
\mathbb{E}
\!\left[\left(1 + \|\psi_i\|^2 + \|\psi_i\|^4\|\theta_i\|^4\right)
\cosh\!\left(b\psi_i^{\top}\theta_i\right)
\right]
<\infty.
\label{eq:linear-light-tailed-boundary-moment}
\end{align}

Then $\widehat \theta_n\xrightarrow{{P}}\theta^{\star}$.
\end{proposition}

\begin{proof}
The proof generalizes the proof of \Cref{thm:linear-main}, provided in Appendix \cref{app:proof-linear-main}.
By the weak law of large numbers and $\mathbb{E}\|\psi_i\|^2<\infty$,
\begin{align}
\widehat Q_n\xrightarrow{{P}}Q.
\label{eq:linear-light-tailed-q-convergence}
\end{align}
Therefore, we have $\widehat Q_n^{-1}\xrightarrow{{P}}Q^{-1}$.

Next, similarly, we consider the decomposition
$
\widehat m_n
=
S_n+R_n,
$
where, recall that,
\begin{align*}
S_n
=
\frac{1}{n}
\sum_{i=1}^n
\psi_iZ_iw_b(T_i),
\qquad
R_n
=
\frac{1}{n}
\sum_{i=1}^n
\psi_iZ_i
\left\{
w_{\widehat b_n}(T_i)-w_b(T_i)
\right\}.
\end{align*}
For the main term, condition on $(\psi_i,\theta_i)$. The scalar unbiasedness identity gives
\begin{align*}
\mathbb{E}
\!\left[
Z_iw_b(T_i)
\mid
\psi_i,\theta_i
\right]
=
\psi_i^{\top}\theta_i.
\end{align*}
Therefore
\begin{align*}
\mathbb{E}[\psi_iZ_iw_b(T_i)]
=
\mathbb{E}[\psi_i\psi_i^{\top}\theta_i] = Q\theta^{\star}.
\end{align*}
The moment assumption implies $\mathbb{E}[\|\psi_i\|^2\|\theta_i\|]<\infty$, which ensures integrability of $\psi_i\psi_i^{\top}\theta_i$. Furthermore, $\phi_iZ_iw_b(T_i)$ is absolute integrable due to
\begin{align*}
    \mathbb{E}\left[\|\psi_i\|\mathbb{E}[|w_b(T_i)|\mid \psi_i,\theta_i]\right] \lesssim \mathbb{E}\left[\|\psi_i\|\cosh(b\psi_i^\top\theta_i)\right]
    < \infty.
\end{align*}
So the weak law of large numbers implies
\begin{align}
S_n\xrightarrow{{P}}Q\theta^{\star}.
\label{eq:linear-light-tailed-main-term-convergence}
\end{align}

It remains to show $R_n\xrightarrow{{P}}0$. Fix $\eta>0$ and let
\begin{align*}
E_{n,\eta}:=\{|\widehat b_n-b|\leq \eta\}.
\end{align*}
By \Cref{prop:light-tailed-richardson-expansion}, $\widehat b_n\xrightarrow{{P}}b$, so $\mathbb{P}(E_{n,\eta})\to 1$. On $E_{n,\eta}$,
\begin{align*}
\|R_n\|
\leq
\frac{1}{n}
\sum_{i=1}^n
\|\psi_i\|
\sup_{|\alpha-b|\leq \eta}
|w_{\alpha}(T_i)-w_b(T_i)|.
\end{align*}
The expectation of the right-hand side converges to zero as $\eta\downarrow 0$ due to 
\begin{align*}
    \mathbb{E}
\!\left[
\|\psi_i\|
\sup_{|\alpha-b|\leq \eta}
\left|
w_{\alpha}(T_i)-w_b(T_i)
\right|
\right] \leq \left(\mathbb{E}
\!\left[
\|\psi_i\|^2\right]\mathbb{E}
\!\left[
\sup_{|\alpha-b|\leq \eta}
\left|
w_{\alpha}(T_i)-w_b(T_i)
\right|^2
\right]\right)^{-1/2} 
\downarrow 0
\end{align*}
by \Cref{lem:envelope}.
Therefore
\begin{align}
R_n\xrightarrow{{P}}0.
\label{eq:linear-light-tailed-residual-convergence}
\end{align}
Combining \eqref{eq:linear-light-tailed-q-convergence}, \eqref{eq:linear-light-tailed-main-term-convergence}, and \eqref{eq:linear-light-tailed-residual-convergence},
\begin{align*}
\widehat \theta_n
=
\widehat Q_n^{-1}\widehat m_n
\xrightarrow{{P}}
Q^{-1}Q\theta^{\star}
=
\theta^{\star},
\end{align*}
which completes the proof.
\end{proof}

If $\psi_i$ is bounded and $\theta_i$ is Gaussian, then \eqref{eq:linear-light-tailed-boundary-moment} holds automatically. Hence Gaussian latent preferences do not require truncation for the boundary-identification argument.

\section{Detailed Experimental Settings}
\label{app:experiments}

This appendix provides the implementation details, hyperparameter settings, and compute cost that were deferred from Section~\ref{sec:experiments}.

\subsection{Common setup}

\paragraph{Boundary estimator.} Throughout, we use the bias-corrected boundary estimator $\btil$ of Eq.~\eqref{eq:btil-def} with $\lambda_n=(\log n)^{3/2}$. 

\paragraph{Weight function evaluation.} The ratio-of-series weight $w_\beta(t)$ of Eq.~\eqref{eq:wb-def} is evaluated by truncating both series at $K=100$ terms. The truncation error decays super-exponentially in $K$: for $t\le T$ and $\beta\ge b_-$, the $k$-th term is bounded by $(2k+1)e^{-(2k+1)^2b_-^2/(2T)}$, which for the parameter ranges used in our experiments gives machine-precision error well before $K=100$. 


\paragraph{First-passage-time sampling.}
Synthetic response times are drawn from the conditional density $f(t\mid v,b)=f_0(t;b)\exp(-v^2 t/2)\cosh(bv)$ by inverse-CDF sampling on a precomputed cache; the real-data experiment uses the recorded $\tau_i$ directly. We tabulate $f$ on the product of an $N_v$-point uniform drift grid and a $998$-point time grid concentrated near the mode of $f_0$ ($\Delta t\approx 5\cdot 10^{-5}$ on $[10^{-4},10^{-2}]$, coarsening to $\Delta t\approx 6.5\cdot 10^{-2}$ in the tail) and truncated at $T_{\max}=20$; the leading large-$t$ asymptote $f_0(t;b)\sim(\pi/2b^2)\exp(-\pi^2 t/(8b^2))$ together with the exponential tilt $\exp(-v^2 t/2)$ bounds the truncated tail mass by $10^{-7}$ for all $|v|\le V_{\max}$ and $b\ge 1.25$ used here. The CDF is built row-wise by the trapezoidal rule and renormalized so that $F(T_{\max}\mid v,b)=1$ exactly; sampling snaps each drift to its nearest $v$-grid point and linearly interpolates in $t$, with bias $O(\Delta v)$ and $O((\Delta t)^2)$ respectively. We use $N_v=500$ over the prior support for the basic experiments and $N_v=801$ over $[-V_{\max},V_{\max}]$ for the contextual experiments, where $V_{\max}=\sum_{k=1}^d\max\bigl(|\theta^\star_k-6\sigma_\theta|,|\theta^\star_k+6\sigma_\theta|\bigr)=12.80$ matches the $\pm 6\sigma_\theta$ truncation of the priors. The cache is built once per scenario and shared across all repetitions; in the contextual setting, if a realized drift ever exceeds $V_{\max}$ during a rep-batch the cache is rebuilt with $V_{\max}\leftarrow\max(1.5\,V_{\max},\,1.1\max_i|V_i|)$, an event that is a deterministic function of the seeded $(\psi_i,\theta_i)$ draws, so the rebuild sequence is identical across re-runs at fixed seed, dtype, and device.

\paragraph{Replications and randomness.} Every MSE and cosine-similarity curve is the Monte-Carlo average over $R=50$ independent replications with independent random seeds. For the synthetic experiments, confidence regions are generated pointwise with half-width $1.96 \cdot \mathrm{std} / \sqrt{R}$; the linear plotting script uses the same formula with the squared-error array. The shaded bands in Figure~\ref{fig:exp_real_cossim} are 95\% bootstrap confidence bands over the $R$ replications.

\paragraph{Compute.} All experiments were run on a single NVIDIA A100-80GB GPU within 10 minutes.

\subsection{Synthetic tabular setting}

We fix $\mu^\star=\E_{\Fstar}[V]=0.25$ and $b^\star=1.25$.

\paragraph{Prior families.}
\begin{itemize}[leftmargin=1.5em,itemsep=1pt,topsep=1pt]
    \item \emph{Uniform:} $V\sim\mathrm{Uniform}[-0.25,0.75]$, which has mean $\mu^\star$ and support width $1$.
    \item \emph{Beta:} $V=\mu^\star-\tfrac{2}{7}+\mathrm{Beta}(2,5)$, a centered Beta rescaled to unit width, yielding a unimodal skewed density.
\end{itemize}


\paragraph{Metric.} The Monte-Carlo MSE $\E[(\muhat-\mu^\star)^2]$, averaged over $R=50$ replications.

\paragraph{Sample-size grid.} $n\in\{10^3,10^4,10^5,10^6\}$.

\subsection{Synthetic linear contextual setting}

We fix $d=4$, $\theta^\star=(0.25,-0.15,0.10,-0.30)$, $\sigma_\theta^2=0.25$, and $b^\star=1.25$. Contexts are $\psi_i\sim\mathrm{Uniform}[-1,1]^d$ i.i.d., and each coordinate of $\theta_i$ is drawn i.i.d.\ from one of the following priors centered at the corresponding coordinate of $\theta^\star$ with variance $\sigma_\theta^2$:

\begin{itemize}[leftmargin=1.5em,itemsep=1pt,topsep=1pt]
    \item \emph{Gaussian:} $\theta_i^{(k)}\sim\cN(\theta^\star_k,\sigma_\theta^2)$, truncated to $\pm 6\sigma_\theta$.
    \item \emph{Uniform:} $\theta_i^{(k)}\sim\mathrm{Uniform}[\theta^\star_k-\sigma_\theta\sqrt{3},\,\theta^\star_k+\sigma_\theta\sqrt{3}]$.
    \item \emph{Beta:} $\theta_i^{(k)}=\theta^\star_k+\sigma_\theta\,\tilde{b}_k$ for a mean-centered $\mathrm{Beta}(2,5)$ variable rescaled to unit variance.
    \item \emph{Laplace:} $\theta_i^{(k)}\sim\mathrm{Laplace}(\theta^\star_k,\sigma_\theta/\sqrt{2})$, truncated to $\pm 6\sigma_\theta$.
\end{itemize}


\paragraph{Metric.} The Monte-Carlo MSE $\E[\|\thetahat-\theta^\star\|_2^2]$, averaged over $R=50$ replications.

\paragraph{Sample-size grid.} $n\in\{10^3,10^4,10^5,10^6,10^7\}$.

\subsection{Real-data intertemporal choice}
\label{app:real-data}

\paragraph{Source.} We use the publicly available dataset of \citet{amasino2019amount} as reanalyzed by \citet{echenique2025response}. Each trial presents a binary choice between a smaller-sooner reward $s_r\in\{0.5,\dots,9.5\}$ dollars paid today ($s_d=0$) and a larger-later reward $\ell_r=10$ dollars paid after delay $\ell_d\in\{1,7,15,30,90,180,365\}$ days. The dataset records the binary choice $Z_i\in\{+1,-1\}$ (encoded $+1$ for larger-later) and the raw reaction time $\tau_i>0$.

\paragraph{Preprocessing.} We drop trials with missing choices or response times, and we drop two participants whose retained trials contain only one choice class (so that the subject-specific Bradley--Terry MLE is degenerate). The resulting replication sample pools $n_{\max}=13{,}793$ valid trials from $S=98$ participants (out of $100$ enrolled). 

\paragraph{Feature encoding.} Following \citet{echenique2025response}, each trial is encoded by the normalized comparison-feature vector $\psi_i=\bigl((\ell_r-s_r)/9.5,\,-(\ell_d-s_d)/365\bigr)\in\R^2$. The first coordinate is the scaled monetary gap and the second is the negated scaled delay gap.


\paragraph{Regularization.}

In 2-D feature space, it is possible that some participants' choices are perfectly linearly separable, which results in divergence of unregularized \btmodel{} MLE. Therefore, we add a $\ell_2$-regularization with penalty parameter $\lambda = 0.1$ to all \btmodel{} MLE losses in the real data setting.

\paragraph{Subsampling.} For each $n$ in the grid $\{10^2,250,500,10^3,2000,5000\}$, we draw $R=50$ random subsamples of size $n$ with replacement from the pooled trials. Each estimator is fit on each subsample.

\paragraph{Metric.} Because the common boundary $b$ is unidentifiable from choices alone, only the direction of $\theta^\star$ is comparable across estimators. We therefore report
\begin{align*}
    \mathrm{CosSim}(\widehat\theta,\theta^\star)
    \;=\; \frac{\widehat\theta^\top\theta^\star}{\|\widehat\theta\|_2\|\theta^\star\|_2}\in[-1,1],
\end{align*}
averaged over the $R=50$ subsamples of size $n$, with $95\%$ bootstrap bands shown as the shaded region in Figure~\ref{fig:exp_real_cossim}.

\paragraph{On the DDM assumption.} The real-data experiment is not an identification claim: human decision-making in intertemporal choice is only approximately modeled by the \ddm{}. The fact that the DDM estimator nonetheless aligns more closely with the subject-level target direction than the pooled \btmodel{} baseline, and continues to improve where the baseline plateaus, is evidence that response times carry information about the aggregate preference direction even when the generative model is potentially misspecified.



\subsection{Ablation: Richardson extrapolation for boundary estimation}
\label{app:bhat-ablation}

We directly compare the two-scale Richardson-extrapolated estimator $\btil$ of Eq.~\eqref{eq:btil-def} against the uncorrected one-scale estimator $\bhat$ of Eq.~\eqref{eq:bhat-def}, with both estimators using the same Laplace-tail rate $\lambda_n=(\log n)^{3/2}$ and the same response-time samples.

\paragraph{Setup.} We reuse the synthetic data-generating processes of Sections~\ref{sec:exp-synthetic} and Appendix~\ref{app:experiments}: tabular data with $b^\star=1.25$ and the prior on $V$ varied across $\{\text{Uniform}, \text{Beta}, \text{Gaussian}, \text{Laplace}\}$, and contextual data with $b^\star=1.25$ and the per-coordinate prior on $\theta_i\in\R^4$ varied across the same four families. For each $n$ we report the median and interquartile range (IQR) over $R=10$ replications.

\paragraph{Results.} Across all priors and both settings (Figures~\ref{fig:bhat_basic} and~\ref{fig:bhat_linear}), the Richardson-extrapolated estimator $\btil$ achieves smaller error and faster convergence to $b^\star=1.25$ than the uncorrected one-scale estimator $\bhat$. This directly substantiates the efficiency-gain of \Cref{thm:two-scale}.

\begin{figure}[t]
  \centering
  \begin{subfigure}[t]{0.47\linewidth}
    \centering
    \includegraphics[width=\linewidth]{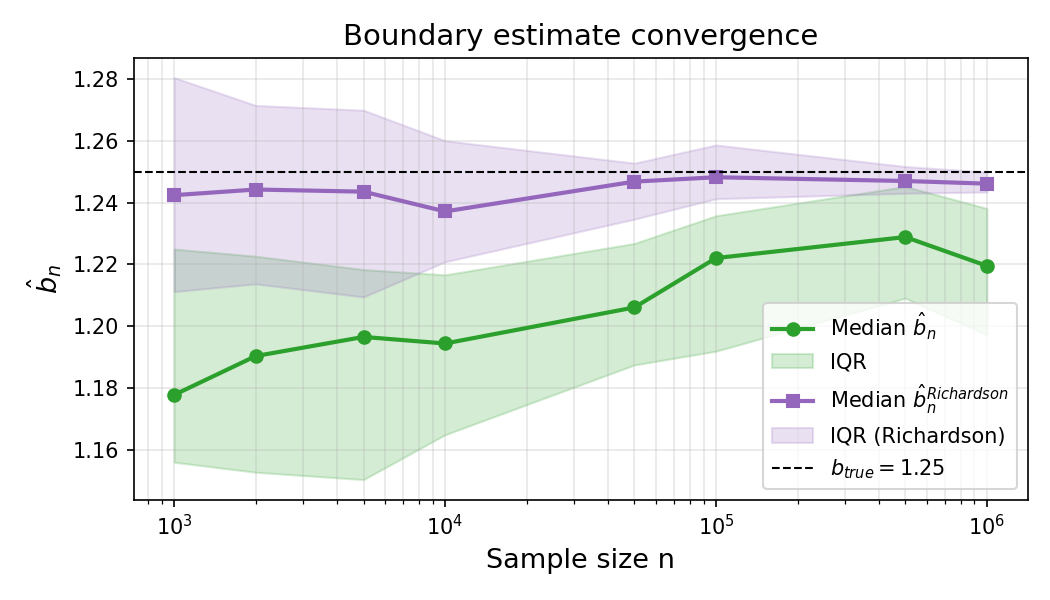}
    \caption{Uniform.}
    \label{fig:bhat_basic_uniform}
  \end{subfigure}
  \hfill
  \begin{subfigure}[t]{0.47\linewidth}
    \centering
    \includegraphics[width=\linewidth]{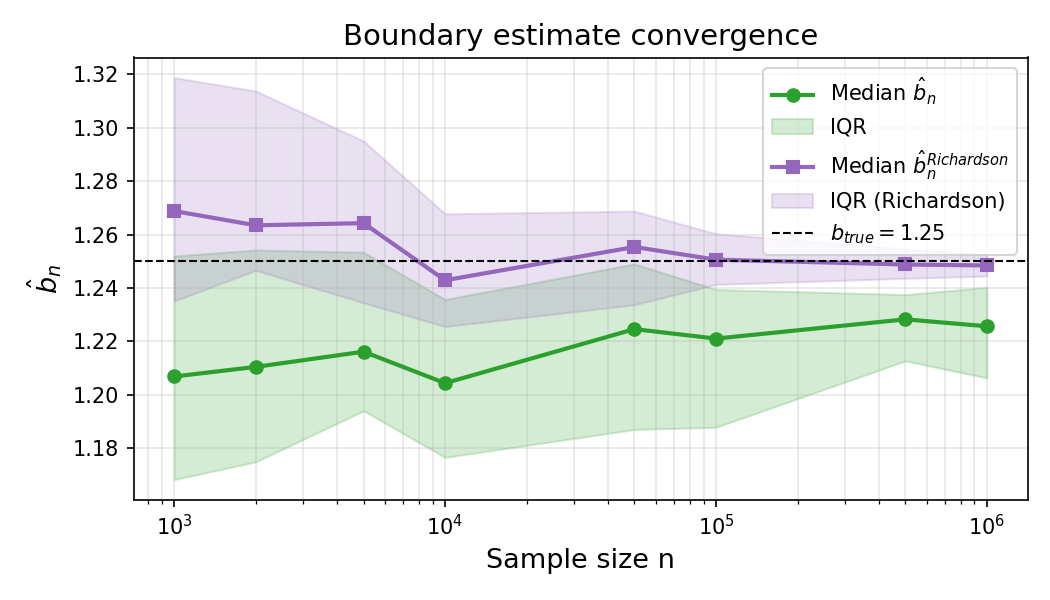}
    \caption{Beta.}
    \label{fig:bhat_basic_beta}
  \end{subfigure}
  \caption{\textbf{Tabular setting.} Median and IQR (over $R=10$ replications) of the one-scale estimator $\bhat$ (green) and the Richardson-extrapolated estimator $\btil$ (purple) as a function of sample size $n$, for each of the two priors on the latent drift $V$. The dashed line marks $b^\star=1.25$. The Richardson correction removes the $O(\lambda_n^{-1/2})$ bias of $\bhat$ across all priors.}
  \label{fig:bhat_basic}
\end{figure}

\begin{figure}[t]
  \centering
  \begin{subfigure}[t]{0.47\linewidth}
    \centering
    \includegraphics[width=\linewidth]{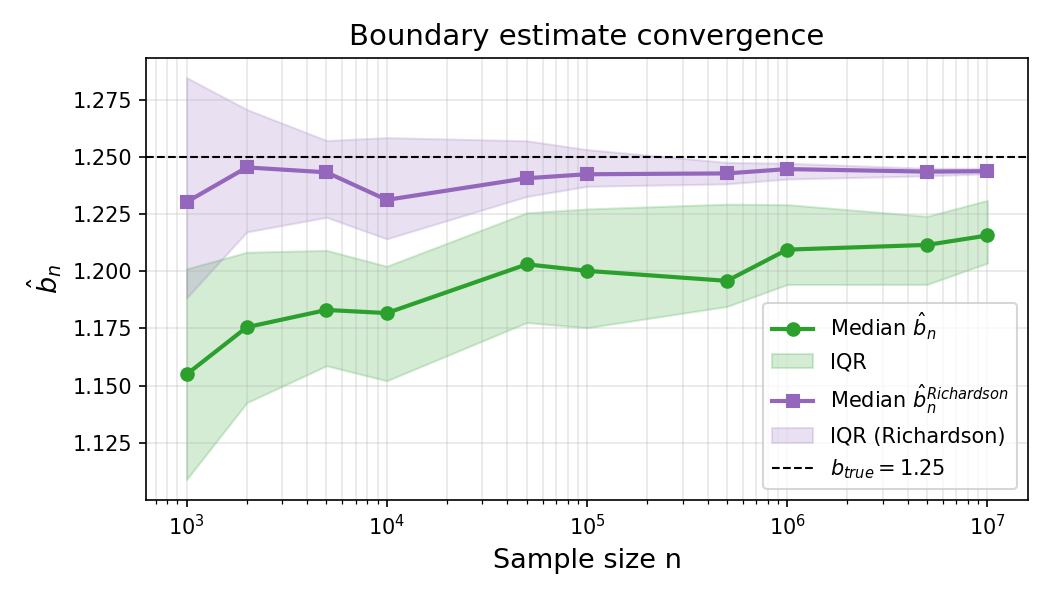}
    \caption{Gaussian.}
    \label{fig:bhat_linear_gaussian}
  \end{subfigure}
  \hfill
  \begin{subfigure}[t]{0.47\linewidth}
    \centering
    \includegraphics[width=\linewidth]{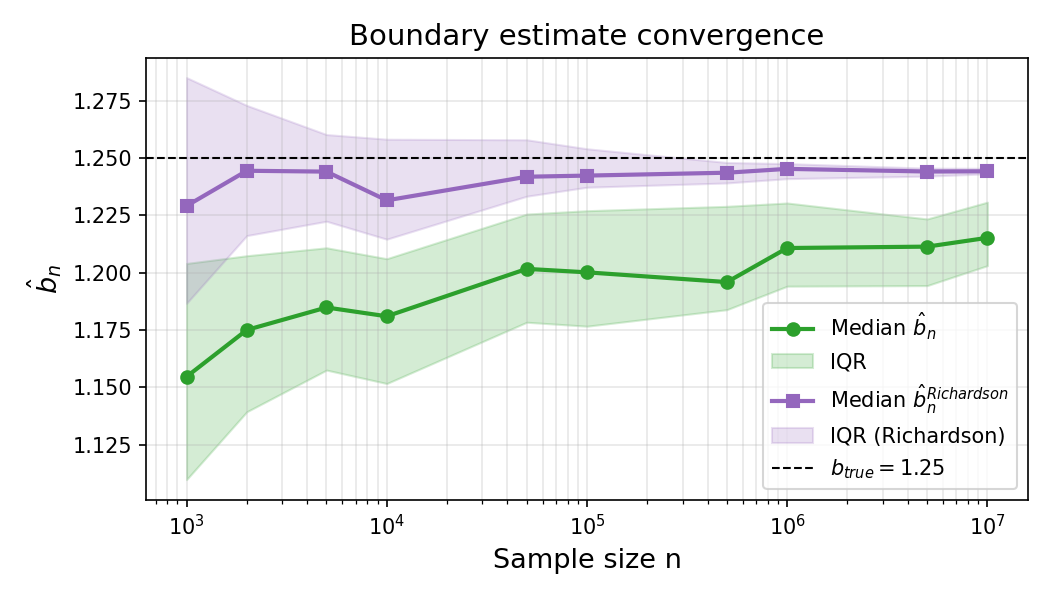}
    \caption{Uniform.}
    \label{fig:bhat_linear_uniform}
  \end{subfigure}
  \\[0.4em]
  \begin{subfigure}[t]{0.47\linewidth}
    \centering
    \includegraphics[width=\linewidth]{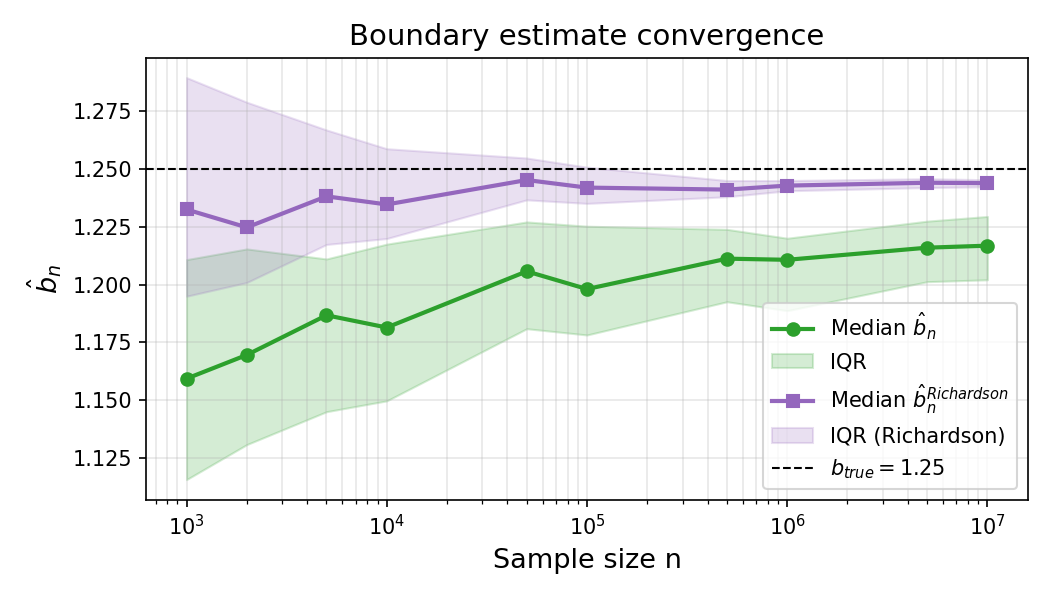}
    \caption{Beta.}
    \label{fig:bhat_linear_beta}
  \end{subfigure}
  \hfill
  \begin{subfigure}[t]{0.47\linewidth}
    \centering
    \includegraphics[width=\linewidth]{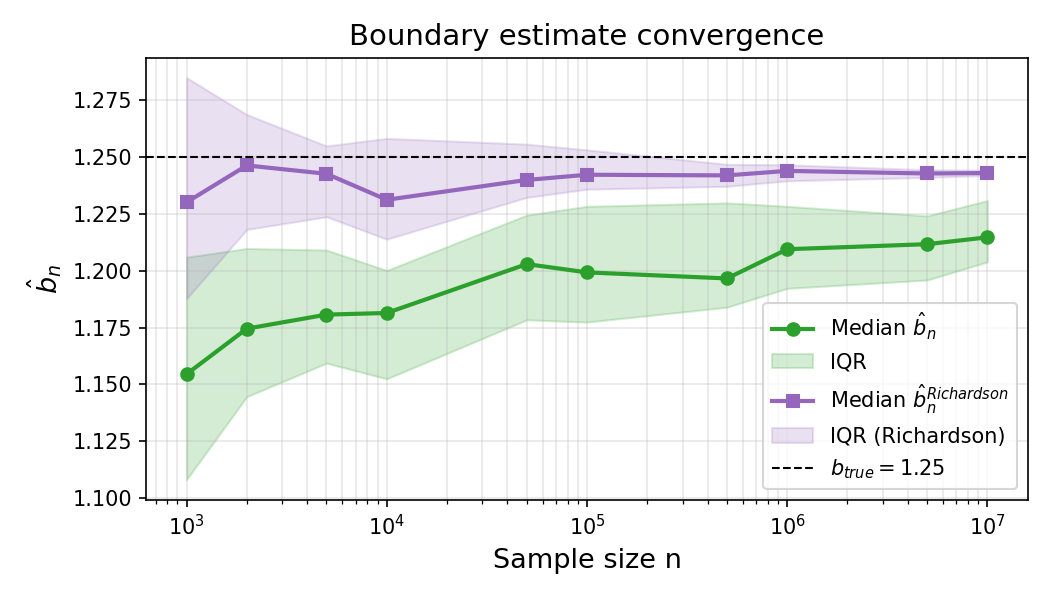}
    \caption{Laplace.}
    \label{fig:bhat_linear_laplace}
  \end{subfigure}
  \caption{\textbf{Linear contextual setting.} Median and IQR (over $R=10$ replications) of $\bhat$ (green) and $\btil$ (purple) versus sample size $n$, for each of the four per-coordinate priors on $\theta_i\in\R^4$. The dashed line marks $b^\star=1.25$. The Richardson correction is also effective in the contextual setting, where the latent drifts $v_i=\psi_i^\top\theta_i$ are heterogeneous through both $\psi_i$ and $\theta_i$.}
  \label{fig:bhat_linear}
\end{figure}

\clearpage

\end{document}